\PassOptionsToPackage{frozencache}{minted}

\IfFileExists{./prepreamble-amspreprint.sty}{\RequirePackage[packages,theorems,changes]{prepreamble-amspreprint}}{}

\makeatletter
\providecommand*{\input@path}{}\edef\input@path{{./scoop-latex/}\input@path}
\makeatother
\PassOptionsToPackage{style = authoryear-comp, backend = biber}{biblatex}
\PassOptionsToPackage{commentmarkup = footnote}{changes}
\RequirePackage{algpseudocode}    

\documentclass[english]{amsart}

\RequirePackage{biblatex}
\RequirePackage{ltxcmds}
\IfFileExists{./preamble-amspreprint.sty}{\RequirePackage[packages,theorems,changes]{preamble-amspreprint}}{}

\usepackage{scoop-local}

\IfFileExists{./postpreamble-amspreprint.sty}{\RequirePackage[packages,theorems,changes]{postpreamble-amspreprint}}{}

\makeatletter
\@ifpackageloaded{changes}{
\definechangesauthor[name = {Manuel Weiß}, color = {red!80!black}]{MW}
\definechangesauthor[name = {Lukas Baumgärtner}, color = {blue!80!black}]{LB}
\definechangesauthor[name = {Roland Herzog}, color = {orange!80!black}]{RH}
\definechangesauthor[name = {Stephan Schmidt}, color = {green!70!black}]{STS}
}{}
\makeatother        

\makeatletter
\@ifpackageloaded{hyperref}{%
	\hypersetup{
		pdftitle = {Geometry Denoising with Preferred Normal Vectors},
		pdfauthor = {Manuel Weiß, Lukas Baumgärtner, Roland Herzog, Stephan Schmidt},
		pdfkeywords = {geometry denoising, normal vectors, split Bregman iteration, assignment function},
		pdfcreator = {Created using the Scoop Template Engine version 1.6.0.}
	}
}{
	\pdfinfo{
		/Title (Geometry Denoising with Preferred Normal Vectors)
		/Author (Manuel Weiß, Lukas Baumgärtner, Roland Herzog, Stephan Schmidt)
		/Subject ()
		/Keywords (geometry denoising, normal vectors, split Bregman iteration, assignment function)
		/Creator (Created using the Scoop Template Engine version 1.6.0.)
	}
}
\makeatother

\title[Geometry Denoising with Preferred Normal Vectors]{Geometry Denoising with Preferred Normal Vectors}

\author[M. Weiß]{Manuel Weiß\orcidlink{0000-0002-6098-9725}}
\address[M. Weiß]{Interdisciplinary Center for Scientific Computing, Heidelberg University, 69120 Heidelberg, Germany}
\email{\detokenize{manuel.weiss@iwr.uni-heidelberg.de}}
\urladdr{https://scoop.iwr.uni-heidelberg.de}

\author[L. Baumgärtner]{Lukas Baumgärtner\orcidlink{0000-0003-1007-4815}}
\address[L. Baumgärtner]{Department of Mathematics, Humboldt University of Berlin, 10099 Berlin, Germany}
\email{\detokenize{lukas.baumgaertner@hu-berlin.de}}
\urladdr{httpss://www.mathematik.hu-berlin.de/en/people/mem-vz/1693318}

\author[R. Herzog]{Roland Herzog\orcidlink{0000-0003-2164-6575}}
\address[R. Herzog]{Interdisciplinary Center for Scientific Computing, Heidelberg University, 69120 Heidelberg, Germany and Institute for Mathematics, Heidelberg University, 69120 Heidelberg, Germany}
\email{\detokenize{roland.herzog@iwr.uni-heidelberg.de}}
\urladdr{https://scoop.iwr.uni-heidelberg.de}

\author[S. Schmidt]{Stephan Schmidt\orcidlink{0000-0002-4888-0794}}
\address[S. Schmidt]{University of Trier, Universitätsring 15, 54296 Trier, Germany}
\email{\detokenize{stephan.schmidt@uni-trier.de}}
\urladdr{https://www.math.uni-trier.de/\string~schmidt}

\thanks{This work was supported by DFG grants HE~6077/10--2 and SCHM~3248/2--2 within the Priority Program SPP~1962 (Non-smooth and Complementarity-based Distributed Parameter Systems: Simulation and Hierarchical Optimization), which is gratefully acknowledged.}

\date{}

\dedicatory{}

\begin{document}

\begin{abstract}
We introduce a new paradigm for geometry denoising using prior knowledge about the surface normal vector.
This prior knowledge comes in the form of a set of preferred normal vectors, which we refer to as label vectors.
A segmentation problem is naturally embedded in the denoising process.
The segmentation is based on the similarity of the normal vector to the elements of the set of label vectors.
Regularization is achieved by a total variation term.
We formulate a split Bregman (ADMM) approach to solve the resulting optimization problem.
The vertex update step is based on second-order shape calculus.
We present various examples including the denoising of an eroded medieval gravestone inscription.

\end{abstract}


\makeatletter
\ltx@ifpackageloaded{hyperref}{%
\subjclass[2010]{\href{https://mathscinet.ams.org/msc/msc2020.html?t=65D18}{65D18}, \href{https://mathscinet.ams.org/msc/msc2020.html?t=49Q10}{49Q10}, \href{https://mathscinet.ams.org/msc/msc2020.html?t=49M15}{49M15}, \href{https://mathscinet.ams.org/msc/msc2020.html?t=65K05}{65K05}, \href{https://mathscinet.ams.org/msc/msc2020.html?t=90C30}{90C30}}
}{%
}
\makeatother

\maketitle

\section{Introduction}
\label{section:Introduction}

Developing efficient computational methods for geometry processing is a core problem in computer graphics.
This includes challenges such as denoising and segmenting surfaces.
In this work, we address both tasks simultaneously by formulating and solving a single optimization problem.
The denoising process is guided by the goal of aligning each surface normal with one of a specified set of preferred normals.
Applications in which we can expect to possess prior knowledge about preferred normal vectors include architecture and crystallography.

We consider surfaces represented by discrete meshes~$\mesh$ with piecewise constant normal vectors.
Generally speaking, denoising requires striking a balance between noise removal and the preservation of important features such as sharp edges \cite{ChenLiWeiWang:2023:1}.
Mesh fairing approaches, \eg, \cite{Taubin:1995:1,DesbrunMeyerSchroederBarr:1999:1}, can give good recovery results but are not designed to preserve sharp features.
$L^0$-minimization \cite{HeSchaefer:2013:1,ZhaoQinZengXuDong:2018:1}, by constrast, can recover sharp features of the geometry, but the non-convex and combinatorial nature of the model may lead to high computational costs.
Total-variation based models \cite{TasdizenWhitakerBurchardOsher:2002:1,BaumgaertnerBergmannHerzogSchmidtVidalNunezWeiss:2025:1} are likewise able to recover sharp features of the geometry, but generally exhibit staircasing effects.

Segmentation aims to partition a given surface into disjoint regions, based on certain features \cite{Shamir:2008:1}.
In particular, it is relatively straightforward to repurpose the variational concept based on an assignment function with total-variation (TV) regularization introduced for image segmentation; see for instance \cite{LellmannKappesYuanBeckerSchnoerr:2009:1}.
We refer the reader to \cite{BaumgaertnerBergmannHerzogSchmidtWeiss:2024:1} for details and for a second, alternative approach based on the TV of the normal vector in place of the TV of the assignment function.

In this paper, we formulate a variational approach that combines the tasks of denoising and segmentation.
We work with triangulated, oriented surfaces embedded in~$\R^3$.
These are represented by meshes~$\mesh$ consisting of triangles~$\triangles$, edges~$\edges$ and vertices~$\vertices$.
Throughout the optimization process, the mesh connectivity is kept fixed but the vertex positions are altered.
The optimization objective combines an application-dependent fidelity term~$\fidelity(\mesh)$, which depends smoothly on the vertex positions (to be specified in \cref{section:numerical-examples}), with a term that promotes the alignment of the triangle-wise normal vector~$\bn_\triangle$ with one of the preferred normal vectors~$\gl$ and simultaneously provides segmentation information.
For convenience, the preferred normal vectors are assumed to have unit length, \ie, $\gl \in \Sphere \coloneqq \setDef{\by \in \R^3}{\abs{\by}_2 = 1}$ holds for $\ell = 1, \ldots, L$.
The segmentation task involves assigning to each triangle $\triangle \in \triangles$ a label corresponding to one of the preferred unit normals~$\gl$, based on the distance of~$\bn_\triangle$ and $\gl$.
The outcome of the segmentation is described by an assignment function $\assign \colon \triangles \to \Delta_L$ mapping each triangle into the probability simplex~$\Delta_L \coloneqq \setDef{\by \in \R^L}{\bone^\transp \by = 1, \; \by \ge \bnull}$.

\subsection*{Contributions}

We propose the following variational framework to achieve the denoising of a surface while incorporating prior knowledge in terms of a set of preferred unit normal vectors~$\gl$, $\ell = 1, \ldots, L$:
\begin{equation}
	\label{eq:general-model}
	\begin{aligned}
		\text{Minimize}
		\;
		&
		\fidelity(\mesh)
		+
		\assignmentWeight \!
		\sum_{\triangle \in \triangles} \abs{\triangle}
		\sum_{\ell=1}^L \doubleindex{\assign} \, \abs{\bn_\triangle - \gl}_2
		+
		\TVweight \TV(\assign)
		\\
		\text{\st}
		\;
		&
		\assign_\triangle \in \Simplex
		\quad
		\text{for all }
		\triangle \in \triangles
		.
	\end{aligned}
\end{equation}
In \eqref{eq:general-model}, the optimization variables are the vertex positions~$\vertexpositions$ of the mesh~$\mesh$, along with the piecewise constant assignment function $\assign \colon \triangles \to \Simplex$.
Since the TV-seminorm \eqref{eq:assignment-function:total-variation} amounts to the $1$-norm of a vector of differences, the \emph{partial} minimization of \eqref{eq:general-model} \wrt $\assign$ is equivalent to a linear optimization problem subject to simplex constraints.
Therefore, $\assign$ will generically attain values in the vertices of the simplex~$\Simplex$, corresponding to an unambiguous assignment of a single label to each triangle.

The first term in \eqref{eq:general-model} is typically a fidelity term that takes into account problem data such as measured vertex positions.
In the second term, $\abs{\triangle}$ denotes the area of a triangle~$\triangle$.
Moreover, $\abs{\bn_\triangle - \gl}_2$ denotes the distance of a triangle's normal vector (depending on the positions of its vertices) to any of the given vectors~$\gl$.
In this paper we use the Euclidean distance in~$\R^3$ for simplicity.
Although the geodesic distance on the unit sphere~$\Sphere$ would be a more appropriate measure of distance, see for instance \cite{LellmannStrekalovskiyKoetterCremers:2013:1,BergmannHerrmannHerzogSchmidtVidalNunez:2020:2}, it does complicate the optimization problem; see, \eg, \cite{BaumgaertnerBergmannHerzogSchmidtVidalNunezWeiss:2025:1}.
Independently of the metric used, the second term in \eqref{eq:general-model} encourages the vertex positions to be altered so that the normal vector $\bn_\triangle$ on each triangle~$\triangle$ coincides \emph{exactly} with one of the preferred normal vectors~$\gl$.
This is due to it acting as an exact penalty function, provided that the value of the assignment function $\assign_\triangle$ is a vertex of the simplex~$\Delta_L$.

The last term in \eqref{eq:general-model} is the total variation of the assignment function~$\assign$, detailed below in \eqref{eq:assignment-function:total-variation}.
It acts as a regularization term that encourages regions of constant assignment to one of the labels, and thus flat regions featuring a constant normal vector.

Problem \eqref{eq:general-model} can be understood as the combination of a mesh denoising approach without preferred normal vector information, based on the model
\begin{equation}
  \label{eq:problem:baseline-tv-based-denoising}
	\text{Minimize}
	\quad
  \fidelity(\mesh)
	+
	\TVweightnormal
  \sum_{\edge\in\edges}
  \abs{\edge}
  \,
  \abs{(\logarithm{\bn_+}{\bn_-}) \cdot \bmu_+}
\end{equation}
considered in \cite{BaumgaertnerBergmannHerzogSchmidtVidalNunezWeiss:2025:1}, and the segmentation problem without denoising based on
\begin{equation}
	\label{eq:segmentation-only}
	\begin{aligned}
		\text{Minimize}
		\quad
		&
		\sum_{\triangle \in \triangles} \abs{\triangle}
		\sum_{\ell=1}^L \doubleindex{\assign} \, \abs{\bn_\triangle - \gl}_2
		+
		\TVweight \TV(\assign)
		\\
		\text{\st}
		\quad
		&
		\assign_\triangle \in \Simplex
		\quad
		\text{for all }
		\triangle \in \triangles
	\end{aligned}
\end{equation}
discussed, \eg, in \cite{LellmannStrekalovskiyKoetterCremers:2013:1,BaumgaertnerBergmannHerzogSchmidtWeiss:2024:1}.
In \eqref{eq:problem:baseline-tv-based-denoising}, $\bn_\pm$ denote the normal vectors on the two sides of an edge~$\edge$, $\bmu_+$ is a unit conormal vector, and $\abs{\edge}$ denotes edge length.

One motivation to consider problems of type \eqref{eq:general-model}, with a set of preferred normal vectors known in advance, arises from crystallography, where the crystalline structure restricts the constellation of possible normal vectors \cite{HsuSadighBertinParkChapmanBulatovZhou:2024:1,FritzenBoehlkeSchnack:2008:1}.
Another application domain where preferred normal vectors are available is architecture, specifically urban mesh denoising \cite{PadenGarciaSanchezLedoux:2022:1,GargalloPeiroFolchRoca:2016:1}, where, \eg, noisy LiDAR (Light Detection and Ranging) measurements of building surfaces are available and require denoising.
We will discuss an example for the latter in~\cref{subsection:example:platonic-solids} and demonstrate that the proposed model \eqref{eq:general-model} outperforms the baseline denoising model \eqref{eq:problem:baseline-tv-based-denoising} that does not utilize preferred normal vector information.

\subsection*{Outline}

In \cref{section:preliminaries}, we introduce the notation and optimization framework for problem \eqref{eq:general-model}.
\Cref{section:admm-scheme} discusses an alternating direction of multipliers (ADMM) scheme to solve problem \eqref{eq:general-model}, and shows how to address the arising subproblems.
In \cref{section:numerical-examples}, we demonstrate the characteristics and performance of the model using four numerical experiments.

\section{Preliminaries}
\label{section:preliminaries}

\subsection*{Triangulated Surfaces}

We consider geometries represented by triangulated, oriented surfaces~$\mesh$ embedded in~$\R^3$.
We denote the set of cells (triangles) by~$\triangles$, the set of edges by~$\edges$ and the set of vertices by~$\vertices$.
We work with manifold meshes without boundary, so that every edge $\edge$ is connected to exactly two triangles.
The triangles on either side of an edge~$\edge$ are denoted by $\edge_+$ and $\edge_-$, respectively, with arbitrary but fixed orientation.
There exists a global outer unit normal vector field~$\bn$ to the surface~$\mesh$, which is constant on each triangle~$\triangle$.

\subsection*{Function Spaces}

We denote by $P_0(\triangle,\R^n)$ the space of constant functions on a triangle~$\triangle$ with values in $\R^n$.
Moreover, we define the space of piecewise constant functions on~$\mesh$ as
\begin{equation*}
	\DG{\triangles}{\R^n}
	\coloneqq
	\setDef[big]
	{\bu \colon \bigcup_{\triangle \in \triangles} \triangle \to \R^n}
	{\bu_\triangle \in P_0(\triangle,\R^n) \text{ for all } \triangle \in \triangles}
	,
\end{equation*}
where $\bu_\triangle$ denotes the constant value of $\bu$ on triangle~$\triangle$.
The ambiguity of function values on edges will be irrelevant.
We will also use $\DG{\triangles}{\Delta_L}$ to denote the subset of $\DG{\triangles}{\R^L}$ with values in the probability simplex $\Delta_L \subseteq \R^L$.

On the skeleton (the union of edges), we analogously define the space
\begin{equation*}
	\DG{\edges}{\R^n}
	\coloneqq
	\setDef[big]
	{\bu \colon \bigcup_{\edge \in \edges} \edge \to \R^n}
	{\bu_\edge \in P_0(\edge,\R^n) \text{ for all } \edge \in \edges}
	.
\end{equation*}
We emphasize the fact that the values $\bu_\triangle$ for~$\bu\in\DG{\triangles}{\R^n}$ remain unchanged when the geometry is deformed, \ie, when the vertex positions~$\vertexpositions$ of the mesh~$\mesh$ are moving.
The same is true for $\DG{\edges}{\R^n}$.

\subsection*{Total Variation}

The \emph{total variation} of $\assign \in \DG{\triangles}{\R^n}$ is given by
\begin{equation}
	\label{eq:assignment-function:total-variation}
	\TV(\assign)
	\coloneqq
	\sum_{\edge \in \edges}
	\abs{\edge}
	\,
	\abs{\eplus{\assign} - \eminus{\assign}}_1^{}
	,
\end{equation}
where $\abs{\edge}$ is the length of the edge~$\edge$ and $\abs{\,\cdot\,}_1$ denotes the $1$-norm of a vector in~$\R^n$; see \cite{LellmannStrekalovskiyKoetterCremers:2013:1,BergmannHerrmannHerzogSchmidtVidalNunez:2020:2}.

\section{ADMM Scheme}
\label{section:admm-scheme}

We recall from \eqref{eq:general-model} the optimization problem under consideration, \ie,
\begin{equation}
	\label{eq:general-model-repeated}
	\text{Minimize}
	\quad
	\fidelity(\mesh)
	+
	\assignmentWeight \!
	\sum_{\triangle \in \triangles} \abs{\triangle}
	\sum_{\ell=1}^L \doubleindex{\assign} \, \abs{\bn_\triangle - \gl}_2
	+
	\TVweight \TV(\assign)
	.
\end{equation}
The unknowns in \eqref{eq:general-model-repeated} are the vertex positions~$\vertexpositions$ of the triangulated surface~$\mesh$ and the assignment function $\assign \in \DG{\triangles}{\Delta_L}$.
Notice that besides the term~$\fidelity(\mesh)$, the triangle areas~$\abs{\triangle}$, edge lengths~$\abs{\edge}$ as well as the normal vector field~$\bn$ in \eqref{eq:general-model-repeated} all depend on~$\vertexpositions$.

In this section, we formulate an alternating direction method of multipliers (ADMM) scheme for \eqref{eq:general-model-repeated}.
We refer the reader to \cite{BoydParikhChuPeleatoEckstein:2010:1,GoldsteinOsher:2009:1} for a background on ADMM and its application to problems involving the total variation.
The reason for using ADMM is that problem \eqref{eq:general-model-repeated} is non-smooth due to the total-variation term $\TV(\assign)$ and the term $\abs{\bn_\triangle - \gl}_2$, and it also has to honor the simplex constraints $\assign_\triangle \in \Simplex$.
At the expense of the introduction of auxiliary variables, ADMM offers an opportunity to decouple these non-smooth terms from the remaining parts of the objective function, allowing us to decompose \eqref{eq:general-model-repeated} into simpler subproblems.

Specifically, we aim to decouple the non-smoothness arising from the total-variation term $\TV(\assign)$ and the assignment penalty $\abs{\bn_\triangle - \gl}_2$, as well as the simplex constraint of the assignment function $\assign$ from the remaining terms of the objective.
For this purpose, we introduce the following auxiliary variables.
The variable $\varW \in \DG{\triangles}{\Simplex}$ is a copy of the assignment function~$\assign$, \ie, it is subject to the constraints $\varW_\triangle = \assign_\triangle$ in~$\R^L$.
However, in contrast to~$\assign$, the simplex constraint $\varW_\triangle \in \Simplex$ will be strictly enforced.
The second auxiliary variable $\varU \in \DG{\triangles}{\R^{L\times 3}}$ is introduced to decouple  the normal vector $\bn_\triangle$ from the non-smooth term $\abs{\bn_\triangle - \gl}_2$.
It is subject to the constraints $\doubleindex{\varU} = \bn_\triangle - \gl$ in~$\R^3$.
Analogously, $\varV \in \DG{\edges}{\R^L}$ is used to decouple the total-variation term from the assignment function~$\assign$.
It is subject to the constraints $\varV_\edge = \eplus{\assign} - \eminus{\assign}$ in~$\R^L$.

Substituting the auxiliary variables~$\varU, \varV, \varW$ into \eqref{eq:general-model-repeated}, we arrive at the augmented problem
\begin{equation}
	\label{eq:general-model-repeated:expanded}
	\begin{aligned}
		\underset{\vertexpositions, \assign, \varU, \varV, \varW}{\text{Minimize}}
		\quad
		&
		\fidelity(\mesh)
		+
		\assignmentWeight \!
		\sum_{\triangle \in \triangles}
		\abs{\triangle}
		\sum_{\ell=1}^L
		\doubleindex{\assign} \, \abs{\doubleindex{\varU}}_2
		+
		\TVWeight \!
		\sum_{\edge \in \edges}
		\abs{\edge}
		\,
		\abs{\varV_\edge}_1
		+
		\sum_{\triangle \in \triangles}
		\characteristic{\Delta_L}(\varW_\triangle)
		\\
		\text{\st}
		\quad
		&
		\paren[auto]\{.{%
			\begin{aligned}
				\doubleindex{\varU}
				&
				=
				\bn_\triangle - \gl \in \R^3
				\quad
				\text{for all }
				\triangle \in \triangles
				\text{ and }
				\ell = 1, \ldots, L
				,
				\\
				\varV_\edge
				&
				=
				\eplus{\assign} - \eminus{\assign} \in \R^L
				\quad
				\text{for all }
				\edge \in \edges
				,
				\\
				\varW_\triangle
				&
				=
				\assign_\triangle \in \R^L
				\quad
				\text{for all }
				\triangle \in \triangles
				.
			\end{aligned}
		}
	\end{aligned}
\end{equation}
Note that we formulate the simplex constraint for $\varW$ in the form of the characteristic function $\characteristic{\Delta_L}$ with values in $\set{0,\infty}$, and we are going to strictly enforce this constraint in every iteration.
We introduce Lagrange multipliers $\dualsimilarity \in \DG{\triangles}{\R^{L\times 3}}$, $\dualjumps \in \DG{\edges}{\R^L}$ and $\dualphiw \in \DG{\triangles}{\R^L}$ for the remaining constraints in \eqref{eq:general-model-repeated:expanded}.

We associate with problem \eqref{eq:general-model-repeated:expanded} the augmented Lagrangian
\begin{align}
	\MoveEqLeft
	\cL_\augm\paren[auto](){\vertexpositions, \assign, \varU, \varV, \varW, \dualsimilarity, \dualjumps, \dualphiw}
	\notag
	\\
	&
	=
	\fidelity(\mesh)
	+
	\assignmentWeight \!
	\sum_{\triangle \in \triangles}
	\abs{\triangle}
	\sum_{\ell=1}^L
	\doubleindex{\assign}
	\,
	\abs{\doubleindex{\varU}
	}_2
	+
	\TVWeight \!
	\sum_{\edge \in \edges}
	\abs{\edge}
	\,
	\abs{\varV_\edge}_1
	+
	\sum_{\triangle \in \triangles}
	\characteristic{\Delta_L}(\varW_\triangle)
	\notag
	\\
	&
	\quad
	+
	\frac{\augm_1}{2}
	\sum_{\triangle \in \triangles}
	\abs{\triangle}
	\sum_{\ell=1}^L
	\abs{\doubleindex{\bn} - \gl - \doubleindex{\varU} + \doubleindex{\dualsimilarity}}_2^2
	\notag
	\\
	&
	\quad
	+
	\frac{\augm_2}{2}
	\sum_{\edge \in \edges}
	\abs{\edge}
	\,
	\abs{\eplus{\assign} - \eminus{\assign} - \varV + \dualjumps_\edge}_2^2
	+
	\frac{\augm_3}{2}
	\sum_{\triangle \in \triangles}
	\abs{\triangle}
	\,
	\abs{\assign_\triangle-\varW_\triangle + \dualphiw_\triangle}_2^2
	.
	\label{eq:augmented-Lagrangian:problem}
\end{align}
The augmentation parameters~$\augm_i > 0$ for $i = 1, 2, 3$ can be chosen separately for each constraint.
In the ADMM scheme, we take turns minimizing the augmented Lagrangian \eqref{eq:augmented-Lagrangian:problem} with respect to one of the primal variables while fixing the others.
We choose to update the variables in the order $\varU, \varV, \varW, \assign, \vertexpositions$.
This exploits the structure of the problem, as the $\varU, \varV, \varW$-subproblems are independent of each other.

In the following subsections, we address how to tackle each subproblem.
While the respective optimization variable does not carry an iteration index, all data to a subproblem will be indexed by $\sequence{\cdot}{k}$---the iteration index of the ADMM scheme---or by $\sequence{\cdot}{k+1}$ in case its update precedes the current subproblem.

\subsection{The \texorpdfstring{$\varU$}{u}-Problem}
\label{subsection:admm:u-problem}

The minimization of \eqref{eq:augmented-Lagrangian:problem} \wrt the variable $\varU \in \DG{\triangles}{\R^{L\times 3}}$ decouples into independent subproblems on each triangle~$\triangle$ and label~$\ell$.
Omitting both the terms from \eqref{eq:augmented-Lagrangian:problem} that do not depend on $\varU$, and canceling the common factor $\abs{\triangle}$, we obtain
\begin{equation}
	\label{eq:admm:u-problem}
	\underset{\doubleindex{\varU} \in \R^3}{\text{Minimize}}
	\quad
	\assignmentWeight
	\,
	\sequence{\doubleindex{\assign}}{k}
	\,
	\abs{\doubleindex{\varU}}_2
	+
	\frac{\augm_1}{2}
	\abs[big]{%
		\doubleindex{\bn}
		-
		\gl
		-
		\doubleindex{\varU}
		+
		\sequence{\doubleindex{\dualsimilarity}}{k}
	}_2^2
	.
\end{equation}
Note that in contrast to most occurrences of this problem in the literature, the coefficient $\assignmentWeight \, \sequence{\doubleindex{\assign}}{k}$ may be negative since the simplex constraint for~$\assign$ is not enforced strictly.
Therefore, problem \eqref{eq:admm:u-problem} may be non-convex.
Nevertheless, global solutions still exist and can be characterized easily, as shown in the following lemma.
We mention that \cite[Lemma~1]{LouYan:2017:1} addresses a related result pertaining to problems involving an additional $\abs{\,\cdot \,}_1$-norm.
\begin{lemma}
	\label{lemma:non-convex-soft-thresholding}
	Suppose that $\gamma \in \R$ and $\bc \in \R^n$, $n \in \N$, are given.
	Then the global solutions $\varU^*$ of
	\begin{equation*}
		\underset{\varU \in \R^n}{\text{Minimize}}
		\quad
		f(\varU)
		\coloneqq
		\gamma
		\,
		\abs{\varU}_2
		+
		\frac{1}{2}
		\abs{\varU-\bc}_2^2
	\end{equation*}
	are given by the soft-thresholding operation
	\begin{equation}
		\label{eq:non-convex-soft-thresholding}
		\varU^*
		=
		\max \set{0, \abs{\bc}_2 - \gamma}
		\cdot
		\begin{cases}
			\frac{\bc}{\abs{\bc}_2}
			&
			\text{if }
			\bc \neq \bnull
			,
			\\
			\be
			&
			\text{if }
			\bc = \bnull
			,
		\end{cases}
	\end{equation}
	where $\be \in \R^n$ is an arbitrary vector of length $\abs{\be}_2 = 1$.
\end{lemma}
\begin{proof}
	In case $\gamma \ge 0$, the result is well-known; see, \eg, \cite[Example~2.16]{CombettesWajs:2005:1}.
	We provide a short proof that is valid for any $\gamma \in \R$.

	We begin with the case $\bc \neq \bnull$.
	Suppose that $\varU \neq \bnull$ is a local minimizer, then the first-order optimality condition
	\begin{equation}
		\label{eq:non-convex-soft-thresholding:first-order-optimality}
		\bnull
		=
		\gamma \, \frac{\varU}{\abs{\varU}_2}
		+
		\paren(){\varU-\bc}
	\end{equation}
	holds.
	This clearly implies $\varU = t \, \bc$ for some $t \in \R \setminus \set{0}$ and leads to the condition
	\begin{equation}
		\label{eq:non-convex-soft-thresholding:first-order-optimality:reduced}
		\bnull
		=
		\gamma \, \frac{\bc}{\abs{\bc}_2} \sgn(t)
		+
		(t - 1) \, \bc
		.
	\end{equation}
	\begin{enumerate}
		\item
			In case $\gamma \ge \abs{\bc}_2$, \eqref{eq:non-convex-soft-thresholding:first-order-optimality:reduced} cannot be satisfied for any $t \neq 0$.
			Consequently, no $\varU \neq \bnull$ can be a local minimizer, and thus $\varU^* = \bnull$ is the unique global minimizer.

		\item
			In case $-\abs{\bc}_2 \le \gamma < \abs{\bc}_2$, \eqref{eq:non-convex-soft-thresholding:first-order-optimality:reduced} is satisfied for the unique value $t_1 = 1 - \frac{\gamma}{\abs{\bc}_2} > 0$.
			Hence, the only candidates for global minimizers are $\varU = \bnull$ and $\varU = t_1 \bc$.
			We compare their objective values:
			\begin{align*}
				f(\bnull)
				&
				=
				\frac{1}{2} \abs{\bc}_2^2
				,
				\\
				f(t_1\bc)
				&
				=
				\gamma \, t_1 \abs{\bc}_2
				+
				\frac{1}{2} (t_1 - 1)^2 \, \abs{\bc}_2^2
				=
				\frac{1}{2} \abs{\bc}_2^2
				-
				\frac{1}{2}
				\paren(){\abs{\bc} - \gamma}^2
				<
				f(\bnull)
				.
			\end{align*}
			Therefore, $\varU^* = t_1\bc$ is the unique global minimizer in this case.

		\item
			In case $\gamma < - \abs{\bc}_2$, \eqref{eq:non-convex-soft-thresholding:first-order-optimality:reduced} is satisfied for the two distinct values $t_2 = 1 + \frac{\gamma}{\abs{\bc}_2} < 0$ and $t_1 > 0$ as above.
			It is easy to see that $f(t \, \bc) < f(-t \, \bc)$ holds for any $t > 0$, hence $t_2$ cannot be a global minimizer.
			Comparing $f(\bnull)$ and $f(t_1 \, \bc)$ as above, we conclude that $f(t_1 \, \bc) < f(\bnull)$ also holds in this case, and thus $\varU^* = t_1 \, \bc$ is the unique global minimizer.
	\end{enumerate}
	The considerations so far confirm \eqref{eq:non-convex-soft-thresholding} in case $\bc \neq \bnull$.
	Now consider the case $\bc = \bnull$, so $f$ is radially symmetric.
	Suppose that $\be$ is an arbitrary unit vector in $\R^n$, then $g(t) \coloneqq f(t \, \be) = \gamma \, \abs{t} + \frac{1}{2} t^2$ holds for any $t \in \R$.
	A simple distinction of cases shows that $t^* = 0$ is the unique global minimizer of~$g$ if $\gamma \ge 0$, while $t^* = \pm \gamma$ are two distinct global minimizers if $\gamma < 0$.
	Since $\be$ was chosen arbitrarily, this shows that the global minimizers of~$f$ are precisely the vectors of the form $\varU^* = -\gamma \, \be$, which confirms \eqref{eq:non-convex-soft-thresholding} also in case $\bc = \bnull$.
\end{proof}

We can now apply \cref{lemma:non-convex-soft-thresholding} to problem \eqref{eq:admm:u-problem} and obtain the update rule
\begin{equation}
	\label{eq:ADMM:u-update}
	\sequence{\doubleindex{\varU}}{k+1}
	=
	\max \set[Big]{0, \abs{\bc}_2 - \frac{\assignmentWeight \, \doubleindex{\assign}}{\augm_1}}
	\cdot
	\begin{cases}
		\frac{\bc}{\abs{\bc}_2}
		&
		\text{if }
		\bc \neq \bnull
		,
		\\
		\be
		&
		\text{if }
		\bc = \bnull
	\end{cases}
\end{equation}
with $\bc \coloneqq \doubleindex{\bn} - \gl + \sequence{\doubleindex{\dualsimilarity}}{k}$ and $\be \in \R^3$ an arbitrary vector of length~$\abs{\be}_2 = 1$.

\subsection{The \texorpdfstring{$\varV$}{v}-Problem}
\label{subsection:admm:v-problem}

The minimization of \eqref{eq:augmented-Lagrangian:problem} \wrt the variable~$\varV \in \DG{\edges}{\R^L}$ decouples into independent subproblems on each edge~$\edge$ and each component (label)~$\ell$.
Omitting the terms that do not depend on~$\varV$ from \eqref{eq:augmented-Lagrangian:problem}, and canceling the common factor~$\abs{\edge}$, the problem on a single edge~$\edge$ and label~$\ell$ reads
\begin{equation}
	\label{eq:admm:v-problem}
	\underset{\doubleindex{\varV}[\edge][\ell] \in \R}{\text{Minimize}}
	\quad
	\TVWeight
	\,
	\abs{\doubleindex{\varV}[\edge][\ell]}
	+
	\frac{\augm_2}{2}
	\paren[big](){%
		\sequence{\elplus{\assign}}{k}
		-
		\sequence{\elminus{\assign}}{k}
		-
		\doubleindex{\varV}[\edge][\ell]
		+
		\sequence{\doubleindex{\dualjumps}[\edge][\ell]}{k}
	}^2
	.
\end{equation}
We can once again apply \cref{lemma:non-convex-soft-thresholding} and obtain
\begin{equation}
	\label{eq:ADMM:v-update}
	\sequence{\doubleindex{\varV}[\edge][\ell]}{k+1}
	=
	\max \set[Big]{%
		\abs[big]{\sequence{\elplus{\assign}}{k} - \sequence{\elminus{\assign}}{k} + \sequence{\doubleindex{\dualjumps}[\edge][\ell]}{k}}
		-
		\frac{\TVWeight}{\augm_2}
		,
		\;
		0
	}
	.
\end{equation}

\subsection{The \texorpdfstring{$\varW$}{w}-Problem}
\label{subsection:admm:w-problem}

The minimization of \eqref{eq:augmented-Lagrangian:problem} \wrt the variable $\varW \in \DG{\triangles}{\Simplex}$ decouples into independent subproblems on each triangle~$\triangle$.
Omitting the terms that do not depend on~$\varW$, and canceling the common factor~$\abs{\triangle}$, the problem on a single triangle~$\triangle$ reads
\begin{equation}
	\underset{\varW_\triangle \in \R^L}{\text{Minimize}}
	\quad
	\characteristic{\Delta_L}(\varW_\triangle)
	+
	\frac{\augm_3}{2}
	\abs[big]{%
		\sequence{\assign_\triangle}{k}
		-
		\varW_\triangle
		+
		\sequence{\dualphiw_\triangle}{k}
	}_2^2
	.
\end{equation}
This is an orthogonal projection problem onto the unit simplex~$\Delta_L$ \wrt the Euclidean inner product, \ie,
\begin{equation}
	\label{eq:ADMM:w-update}
	\sequence{\varW_\triangle}{k+1}
	=
	\proj_{\Delta_L}\paren[big](){%
		\sequence{\assign_\triangle}{k}
		+
		\sequence{\dualphiw_\triangle}{k}
	}
\end{equation}
holds.
It can be solved efficiently, \eg, using \cite[Algorithm~1]{WangCarreiraPerpinan:2013:1}.
The main idea of that algorithm is to consider variables of the form $\sequence{\doubleindex{\varW}}{k+1} = \max \set{0, \, \eta + \sequence{\doubleindex{\assign}}{k} + \sequence{\doubleindex{\dualphiw}}{k}}$ with a suitable $\eta \in \R$.
The optimal value of the shift parameter~$\eta$ can be found by sorting the entries of the vector $\sequence{\assign_\triangle}{k} + \sequence{\dualphiw_\triangle}{k}$.

\subsection{The \texorpdfstring{$\assign$}{phi}-Problem}
\label{subsection:admm:phi-problem}

The minimization of \eqref{eq:augmented-Lagrangian:problem} \wrt the variable $\assign \in \DG{\triangles}{\Delta_L}$ leads to the problem
\begin{multline}
	\label{eq:ADMM:phi-update}
	\underset{\assign \in \DG{\triangles}{\R^L}}{\text{Minimize}}
	\quad
	\assignmentWeight \!
	\sum_{\triangle \in \triangles}
	\abs{\triangle}
	\sum_{\ell=1}^L
	\doubleindex{\assign}
	\,
	\abs[big]{\sequence{\doubleindex{\varU}}{k+1}}_2
	\\
	+
	\frac{\augm_2}{2}
	\sum_{\edge \in \edges}
	\abs{\edge}
	\,
	\abs[big]{%
		\eplus{\assign}
		-
		\eminus{\assign}
		-
		\sequence{\varV}{k+1}
		+
		\sequence{\dualjumps_\edge}{k}
	}_2^2
	+
	\frac{\augm_3}{2}
	\sum_{\triangle \in \triangles}
	\abs{\triangle}
	\,
	\abs[big]{%
		\assign_\triangle
		-
		\sequence{\varW_\triangle}{k+1}
		+
		\sequence{\dualphiw_\triangle}{k}
	}_2^2
	,
\end{multline}
where again we omitted all terms from \eqref{eq:augmented-Lagrangian:problem} that do not depend on~$\assign$.
This is a smooth problem that decouples \wrt the components $\assign_{\cdot,\ell}$ but remains spatially coupled.
The coupling is sparse, however, and it reflects the triangle connectivity structure across edges.
Since the objective in \eqref{eq:ADMM:phi-update} is a strongly convex quadratic polynomial with sparse Hessian, we solve the resulting linear system using the \petsc implementation \cite{DalcinPazKlerCosimo:2011:1,BalayAbhyankarAdamsBrownBruneBuschelmanConstantinescuDenerFaibussowitschGroppIsaacKaushikKnepleyKongMcInnesMunsonRuppSananSarichSmithZhangZhangBensonSuhDalcinEijkhoutHaplaJolivetKarpeevKrugerMayMitchellRomanZampiniMillsZhang:2024:1} of the conjugate gradient (CG) method, starting from an all-zero initial guess.
As stopping criterion, we use a relative tolerance $\rtol = 10^{-2}$.
All other settings were left at their default values.

\subsection{The Shape Optimization Problem}
\label{subsection:admm:shape-optimization-problem}

The final optimization step in each ADMM iteration is to update the vertex positions of the mesh~$\mesh$.
Provided that the mesh remains non-degenerate (\ie, a manifold), which we monitor throughout the iterations, the augmented Lagrangian~\eqref{eq:augmented-Lagrangian:problem} depends smoothly on the vertex coordinates~$\vertexpositions$.
Nevertheless, this is the most complex subproblem in the ADMM scheme.
Notice that besides the fidelity term $\fidelity(\mesh)$,
the triangle areas~$\abs{\triangle}$, the edge lengths~$\abs{\edge}$ as well as the normal vector field~$\bn$ all depend nonlinearly on the vertex coordinates~$\vertexpositions$.
Therefore, all terms in \eqref{eq:augmented-Lagrangian:problem} need to be considered in the minimization \wrt the mesh coordinates~$\vertexpositions$.

In fact, the minimization of \eqref{eq:augmented-Lagrangian:problem} \wrt~$\vertexpositions$ can be considered a discretized shape optimization problem.
We treat it as such and use a globalized shape Newton scheme for its approximate solution.
Any motion of the vertices generates a piecewise linear, continuous deformation field $\trialfunction \in \CG{\mesh}{\R^3}$.
In each iteration, we find the deformation field $\trialfunction \in \CG{\mesh}{\R^3}$ from the shape Newton system
\begin{multline}
	\label{eq:admm:shape-optimization-problem:Newton-equation}
	\d^2 \cL_\augm
	\paren[big](){%
		\vertexpositions
		,
		\sequence{\assign}{k+1}
		,
		\sequence{\varU}{k+1}
		,
		\sequence{\varV}{k+1}
		,
		\sequence{\varW}{k+1}
		,
		\sequence{\dualsimilarity}{k}
		,
		\sequence{\dualjumps}{k}
		,
		\sequence{\dualphiw}{k}
	}
	\paren{[}{]}{\trialfunction,\testfunction}
	\\
	=
	-
	\d \cL_\augm
	\paren[big](){%
		\vertexpositions
		,
		\sequence{\assign}{k+1}
		,
		\sequence{\varU}{k+1}
		,
		\sequence{\varV}{k+1}
		,
		\sequence{\varW}{k+1}
		,
		\sequence{\dualsimilarity}{k}
		,
		\sequence{\dualjumps}{k}
		,
		\sequence{\dualphiw}{k}
	}
	\paren{[}{]}{\testfunction}
	,
\end{multline}
where $\testfunction \in \CG{\mesh}{\R^3}$ is an arbitrary test function.
Similarly to~\eqref{eq:ADMM:phi-update}, we solve \cref{eq:admm:shape-optimization-problem:Newton-equation} using the \petsc implementation of the CG method.
We note that the Newton matrix in \eqref{eq:admm:shape-optimization-problem:Newton-equation} is not necessarily positive definite.
Therefore, the (truncated) CG iteration stops in case a direction of negative curvature is encountered, as described in \cite[Chapter~7.1]{NocedalWright:2006:1}.
Further, we use a preconditioner realized by an incomplete Cholesky decomposition of the representation matrix of the inner product
\begin{equation}
	\label{eq:deformation-field:inner-product}
	\inner{\trialfunction}{\testfunction}_{\mesh}
	=
	\int_\mesh
	\trialfunction\cdot\testfunction
	\,
	\d \bx
	+
	c
	\int_\mesh
	D_\mesh \trialfunction
	\dprod
	D_\mesh \testfunction
	\,
	\d \bx
\end{equation}
with a suitable parameter $c > 0$.
The value for~$c$ is specified with each numerical experiment in \cref{section:numerical-examples}.
The CG iterations continue until either the residual norm is small enough, or a search direction of negative curvature is detected.

The approximate shape Newton direction~$\trialfunction$ is then compared to the shape gradient direction \wrt the inner product \eqref{eq:deformation-field:inner-product}, and the shape gradient is used as a fallback direction.
Finally, an Armijo line search procedure is applied to find a suitable step size~$t > 0$ in the vertex position update $\sequence{\vertexpositions}{k+1} \gets \sequence{\vertexpositions}{k} + t \, \trialfunction$.
For details, we refer the reader to \cite[Algorithm~4.3]{BaumgaertnerBergmannHerzogSchmidtVidalNunezWeiss:2025:1}.

\subsection{Overall ADMM Scheme}
\label{subsection:admm:overall-scheme}

The overall ADMM scheme for the solution of \eqref{eq:general-model-repeated} by way of its augmented variant \eqref{eq:general-model-repeated:expanded} is summarized as \cref{algorithm:ADMM}.
We consider the algorithm converged if the absolute change in all eight variables is small between two successive iterations.

\begin{algorithm}[htb]
	\caption{ADMM for problem \eqref{eq:general-model-repeated} with fidelity term \eqref{eq:fidelity-term}.}
	\label{algorithm:ADMM}
	\begin{algorithmic}[1]
		\Require
		preferred normal vectors (labels) $\labels_1, \ldots, \labels_L \in \Sphere$
		\Require
		vertex positions $\vertexdata \in \CG{\mesh}{\R^3}$ for the fidelity term \eqref{eq:fidelity-term}
		\Require
		assignment parameter $\assignmentWeight > 0$,
		TV penalty parameter~$\TVweight > 0$
		\Require
		mesh~${\mesh}$ with initial vertex positions~$\sequence{\vertexpositions}{0}$
		\Require
		initial assignment function
		$\sequence{\assign}{0} \in \DG{\triangles}{\R^L}$
		\Require
		initial multiplier estimates
		$\sequence{\dualsimilarity}{0}$,
		$\sequence{\dualjumps}{0}$,
		$\sequence{\dualphiw}{0}$
		\Require
		augmentation parameter~$\augm_1,\augm_2,\augm_3 > 0$,
		inner product parameter~$c > 0$
		\Ensure
		approximate solution of \eqref{eq:general-model-repeated}, respectively \eqref{eq:general-model-repeated:expanded}
		\State{Set $k \coloneqq 0$}
		\While{not converged}
		\State{Set $\sequence{\varU}{k+1}$ using \eqref{eq:ADMM:u-update}}
		\label{step:admm:u-update}
		\State{Set $\sequence{\varV}{k+1}$ using \eqref{eq:ADMM:v-update}}
		\label{step:admm:v-update}
		\State{Set $\sequence{\varW}{k+1}$ using \eqref{eq:ADMM:w-update} }
		\label{step:admm:w-update}
		\State Find an approximate minimizer $\sequence{\assign}{k+1}$ of \eqref{eq:ADMM:phi-update}; see \cref{subsection:admm:phi-problem}
		\label{step:admm:phi-update}
		\State Find an approximate minimizer $\sequence{\vertexpositions}{k+1}$ of \eqref{eq:augmented-Lagrangian:problem}; see \cref{subsection:admm:shape-optimization-problem}
		\label{step:admm:shape-update}
		\Statex Update the Lagrange multipliers for all $\triangle \in \triangles$, $\edge \in \edges$ and $\ell = 1, \ldots, L$:
		\State
		\begin{math}
			\sequence{\doubleindex{\dualsimilarity}}{k+1}
			\gets
			\sequence{\doubleindex{\dualsimilarity}}{k}
			+
			\sequence{\bn}{k+1}
			-
			\gl
		\end{math}
		\label{step:admm:dualsimilarity-update}
		\State
		\begin{math}
			\sequence{\dualphiw_\triangle}{k+1}
			\gets
			\sequence{\dualphiw_\triangle}{k}
			+
			\sequence{\assign_\triangle}{k+1}
			-
			\sequence{\varW_\triangle}{k+1}
		\end{math}
		\label{step:admm:dualphiw-update}
		\State
		\begin{math}
			\sequence{\dualjumps_\edge}{k+1}
			\gets
			\sequence{\dualjumps_\edge}{k}
			+
			\sequence{\eplus{\assign}}{k+1}
			-
			\sequence{\eminus{\assign}}{k+1}
			-
			\sequence{\varV_\edge}{k+1}
		\end{math}
		\State{Set $k\coloneqq k+1$}
		\EndWhile
	\end{algorithmic}
\end{algorithm}

\section{Numerical Experiments}
\label{section:numerical-examples}

In this section, we present four numerical experiments using the proposed problem \eqref{eq:general-model-repeated}.
Each experiment features a distinct initial geometry (represented by a surface mesh~$\mesh$) and label set of preferred normal vectors and is designed to showcase a different aspect.

The first example (\cref{subsection:example:sphere}) clarifies the effect of the parameters $\assignmentWeight$ (assignment weight) and $\TVweight$ (TV weight).
We therefore use a sphere with noisy vertex positions as a simple model geometry.
In the second example (\cref{subsection:example:platonic-solids}), we show that the model is capable of affecting large deformations of the initial geometry, provided that the assignment weight~$\assignmentWeight$ is sufficiently large.
To demonstrate this, we deform a sphere mesh into one of several platonic solids, which feature only a few distinct preferred normal vectors.
In the third example (\cref{subsection:example:city-skyline}), we consider denoising a noisy, synthetic city skyline scan, illustrating how knowledge of preferred normal vectors can substantially enhance the reconstruction quality.
We use problem \eqref{eq:general-model-repeated} for an artistic purpose and apply it to the Stanford bunny (\cref{subsection:example:stanford-bunny}) to achieve a rough wood carving effect.
In the final example (\cref{subsection:example:inscription}), we apply the proposed model \eqref{eq:general-model-repeated} to a 3D scan of an inscription on a gravestone.
Knowledge about the limited number of facet angles occurring in the inscription can help decipher the text, which is partially eroded and covered by moss.

The fidelity and regularization terms
\begin{equation}
	\label{eq:fidelity-term}
	\fidelity(\vertexpositions)
	=
	\fidelity_1(\vertexpositions;\vertexdata)
	+
	\meshqualityweight \,
	\fidelity_2(\vertexpositions)
\end{equation}
are used in all examples, where
\begin{equation}
	\fidelity_1(\vertexpositions;\vertexdata)
	=
	\sum_{\vertex \in \vertices}
	\abs{\vertexpositions_\vertex-\vertexdata_\vertex}_2^2
	\quad
	\text{and}
	\quad
	\fidelity_2(\vertexpositions)
	=
	\sum_{\triangle \in \triangles}
	\frac{1}{\abs{\triangle}}
	.
\end{equation}
Here $\vertexdata_\vertex \in \R^3$ denotes given, generally noisy data for each vertex $\vertex \in \vertices$.
Simultaneously, $\vertexdata$ serves as the initial vertex positions~$\sequence{\vertexpositions}{0}$ in \cref{algorithm:ADMM}.
The second addend in \eqref{eq:fidelity-term} is a rudimentary mesh quality term that prevents triangles from becoming too small.

All experiments were implemented in \fenics version~2019.1.0 \cite{AlnaesBlechtaHakeJohanssonKehletLoggRichardsonRingRognesWells:2015:1}.

\subsection{Sphere Example}
\label{subsection:example:sphere}

This example serves to study the effect of the assignment weight~$\assignmentWeight$ and the TV weight~$\TVweight$ on the solution of problem \eqref{eq:general-model-repeated}.
We use a discretization of the unit sphere~$\Sphere \coloneqq \setDef{\bx \in \R^3}{\abs{\bx}_2 = 1}$ obtained by \mshr (the built-in \fenics mesh generation library) into \num{2010}~triangles and \num{1007}~vertices as a simple model geometry; see \cref{figure:sphere:noisy-data}.
We then add Gaussian noise to the vertex positions $\vertexdata$.
The vertex-dependent variance is chosen to be $\sigma^2 = 0.01 \, e^2$, where $e$ denotes the average length of the incident edges.
Further, we choose $L = 20$~preferred normal vectors~$\gl$ distributed uniformly around the sphere by means of the Fibonacci lattice; see \cite{Gonzalez:2009:1}.
The model and algorithmic parameters are shown in \cref{table:sphere:parameters}.

\begin{table*}[htb]
  \centering
  \begin{tabular}{lSSSSS}
		\toprule
    &
		{\mrep[r]{\cref{figure:sphere:alpha-small-beta-small,figure:sphere:alpha-small-beta-medium,figure:sphere:alpha-small-beta-large}}{\ref{figure:sphere:alpha-small-beta-small}}}
    &
		{\labelcref{figure:sphere:alpha-medium-beta-small,figure:sphere:alpha-medium-beta-medium,figure:sphere:alpha-medium-beta-large}}
    &
    {\labelcref{figure:sphere:alpha-large-beta-small}}
    &
    {\labelcref{figure:sphere:alpha-large-beta-medium}}
    &
    {\labelcref{figure:sphere:alpha-large-beta-large}}
		\\
		\midrule
    mesh quality weight~$\meshqualityweight$ \eqref{eq:fidelity-term}
		&
    \num{e-6}
    &
    \num{e-6}
    &
    \num{e-6}
    &
    \num{e-6}
    &
    \num{e-6}
		\\
    augmentation parameter~$\augm_1$ \eqref{eq:augmented-Lagrangian:problem}
		&
    \num{0.2}
    &
    \num{2}
    &
    \num{2}
    &
    \num{2}
    &
    \num{10}
		\\
    augmentation parameter~$\augm_2$ \eqref{eq:augmented-Lagrangian:problem}
		&
    \num{0.2}
    &
    \num{2}
    &
    \num{2}
    &
    \num{2}
    &
    \num{10}
		\\
    augmentation parameter~$\augm_3$ \eqref{eq:augmented-Lagrangian:problem}
		&
    \num{0.2}
    &
    \num{2}
    &
    \num{2}
    &
    \num{2}
    &
    \num{10}
		\\
    inner product parameter~$c$ \eqref{eq:deformation-field:inner-product}
    &
    \num{0.1}
    &
    \num{0.1}
    &
    \num{0.1}
    &
    \num{0.1}
    &
    \num{0.1}
    \\
		\bottomrule
  \end{tabular}
	\caption{%
		Parameters for the sphere example (\cref{figure:sphere}).
		The values of the assignment weight~$\assignmentWeight$ and total-variation weight~$\TVweight$ parameters \eqref{eq:general-model-repeated} are given in the caption of \cref{figure:sphere}.
	}
  \label{table:sphere:parameters}
\end{table*}

\begin{figure}[htb]
	\centering
	\includegraphics[width = 0.5\textwidth]{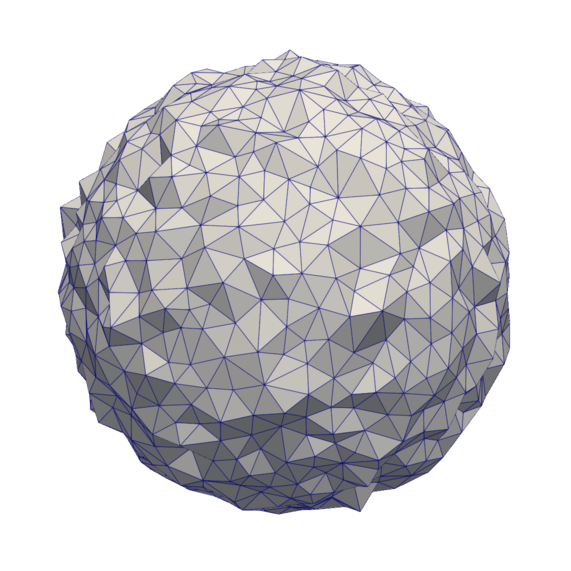}
	\caption{Noisy input data for the sphere geometry experiment (\cref{subsection:example:sphere}).}
	\label{figure:sphere:noisy-data}
\end{figure}

\Cref{figure:sphere} shows the solution of problem \eqref{eq:general-model-repeated} for different values of the assignment weight~$\assignmentWeight$ and the TV weight~$\TVweight$.
We observe that as the assignment weight~$\assignmentWeight$ increases (top to bottom), the normal vectors align more and more with one of the preferred normal directions while simultaneously the noise is removed.
With increasing total-variation weight~$\TVweight$ (left to right), regions of constant assignments grow larger, eventually leading to only a subset of the preferred normal vectors to be used.
For instance, the solution $\TVweight = 1$, $\assignmentWeight = 0.1$ (\cref{figure:sphere:alpha-large-beta-large}, bottom right) uses only~$14$ of the $L = 20$~labels available.

\begin{figure*}[htp]
	\centering
	\begin{subfigure}[b]{0.3\textwidth}
		\centering
		\includegraphics[width = \textwidth]{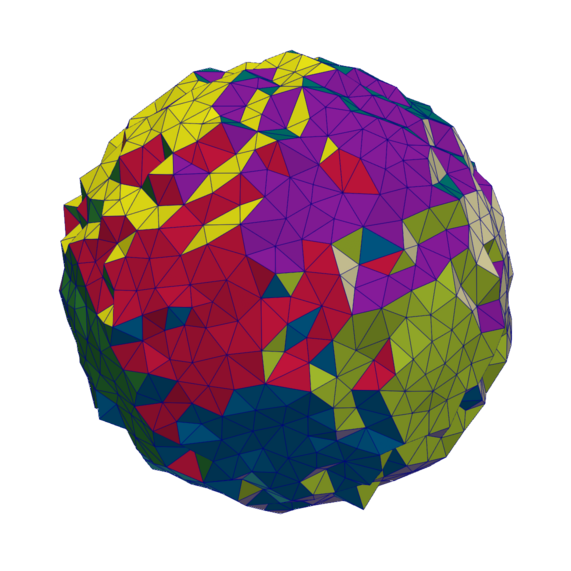}
		\caption{$\assignmentWeight = 0.1$, $\TVweight = 0.0001$}
		\label{figure:sphere:alpha-small-beta-small}
	\end{subfigure}
	\hfill
	\begin{subfigure}[b]{0.3\textwidth}
		\centering
		\includegraphics[width = \textwidth]{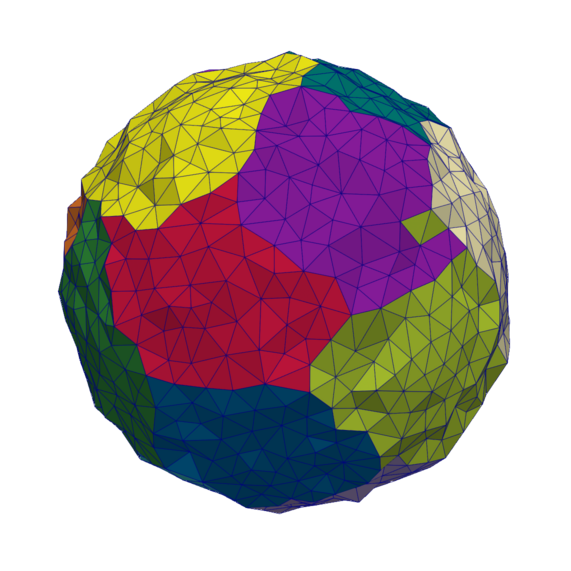}
		\caption{$\assignmentWeight = 0.1$, $\TVweight = 0.001$}
		\label{figure:sphere:alpha-small-beta-medium}
	\end{subfigure}
	\hfill
	\begin{subfigure}[b]{0.3\textwidth}
		\centering
		\includegraphics[width = \textwidth]{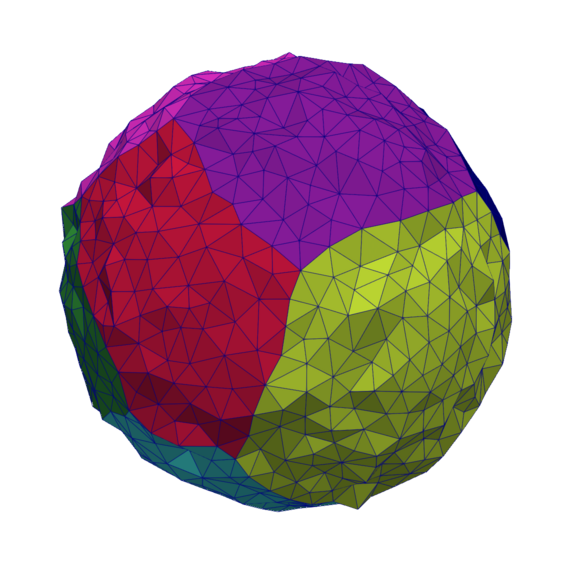}
		\caption{$\assignmentWeight = 0.1$, $\TVweight = 0.01$}
		\label{figure:sphere:alpha-small-beta-large}
	\end{subfigure}
	\hfill
	\begin{subfigure}[b]{0.3\textwidth}
		\centering
		\includegraphics[width = \textwidth]{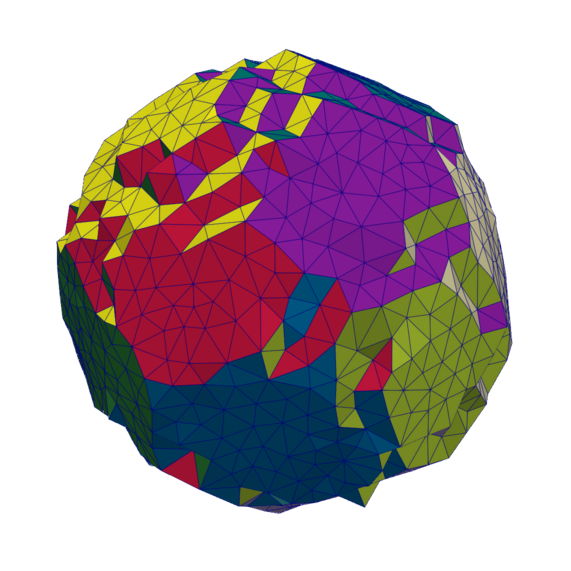}
		\caption{$\assignmentWeight = 0.3$, $\TVweight = 0.0003$}
		\label{figure:sphere:alpha-medium-beta-small}
	\end{subfigure}
	\hfill
	\begin{subfigure}[b]{0.3\textwidth}
		\centering
		\includegraphics[width = \textwidth]{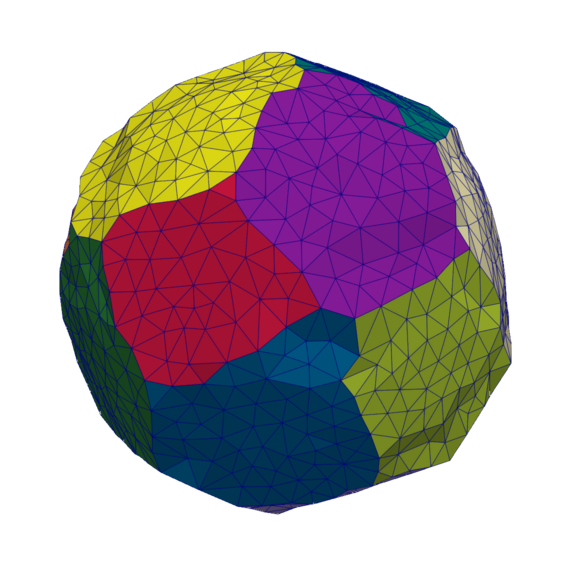}
		\caption{$\assignmentWeight = 0.3$, $\TVweight = 0.003$}
		\label{figure:sphere:alpha-medium-beta-medium}
	\end{subfigure}
	\hfill
	\begin{subfigure}[b]{0.3\textwidth}
		\centering
		\includegraphics[width = \textwidth]{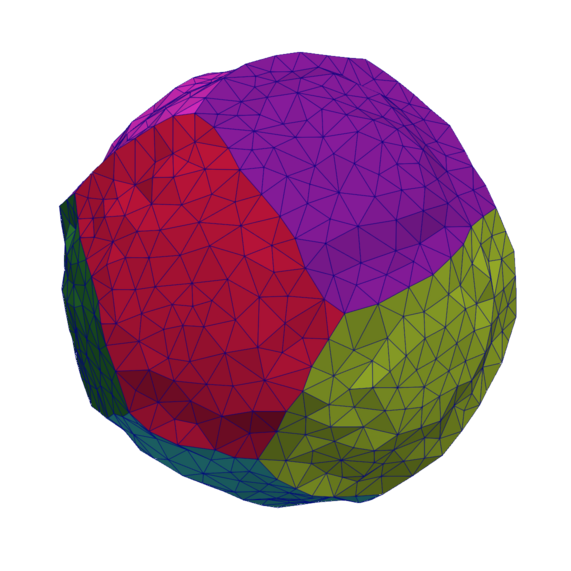}
		\caption{$\assignmentWeight = 0.3$, $\TVweight = 0.03$}
		\label{figure:sphere:alpha-medium-beta-large}
	\end{subfigure}
	\hfill
	\begin{subfigure}[b]{0.3\textwidth}
		\centering
		\includegraphics[width = \textwidth]{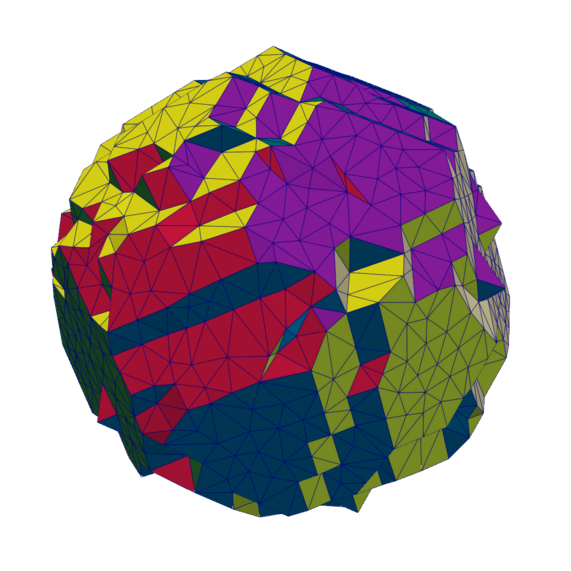}
		\caption{$\assignmentWeight = 1$, $\TVweight = 0.001$}
		\label{figure:sphere:alpha-large-beta-small}
	\end{subfigure}
	\hfill
	\begin{subfigure}[b]{0.3\textwidth}
		\centering
		\includegraphics[width = \textwidth]{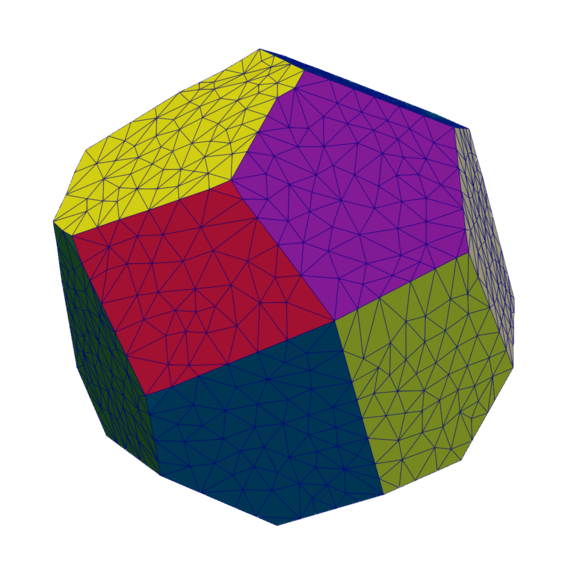}
		\caption{$\assignmentWeight = 1$, $\TVweight = 0.01$}
		\label{figure:sphere:alpha-large-beta-medium}
	\end{subfigure}
	\hfill
	\begin{subfigure}[b]{0.3\textwidth}
		\centering
		\includegraphics[width = \textwidth]{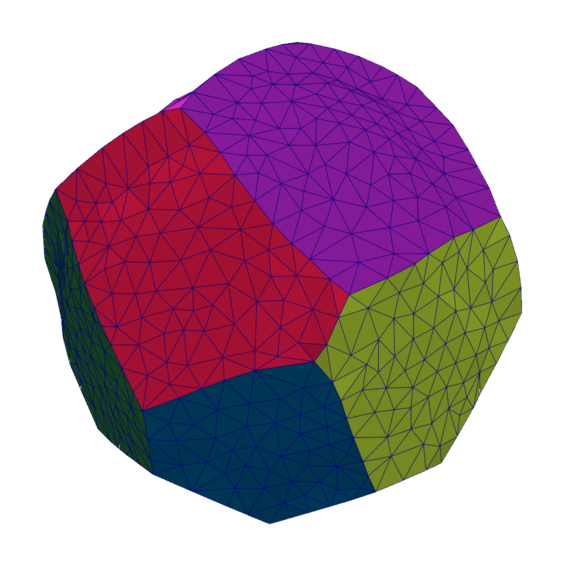}
		\caption{$\assignmentWeight = 1$, $\TVweight = 0.1$}
		\label{figure:sphere:alpha-large-beta-large}
	\end{subfigure}
	\caption{%
		Solutions for problem \eqref{eq:general-model-repeated} for the sphere example (\cref{subsection:example:sphere}, $L = 20$ labels) with noise as shown in \cref{figure:sphere:noisy-data} for different values of the assignment weight~$\assignmentWeight$ and the TV weight~$\TVweight$.
		Cells are colored according to the assigned label.
	}
	\label{figure:sphere}
\end{figure*}

\subsection{Platonic Solids}
\label{subsection:example:platonic-solids}

In this section, we demonstrate that the model \eqref{eq:general-model-repeated} is capable of affecting large deformations of the initial geometry, provided that the assignment weight~$\assignmentWeight$ is sufficiently large.
To demonstrate this, we choose a setup in which a sphere mesh is deformed into one of several platonic solids, which feature only a few distinct preferred normal vectors~$\gl$.
We use a discretization of the unit sphere~$\Sphere$ into \num{2601}~vertices and \num{5198}~triangles obtained again using \mshr.
No noise is added to the resulting vertex positions~$\vertexdata$.
The labels $\gl$ are chosen as the normal vectors of certain platonic solids.
Specifically, we consider the tetrahedron ($L = 4$), and the dodecahedron ($L = 12$) as examples.
We use a large value $\assignmentWeight = 20$ for the the assignment weight, so that the surface is forced to align with the preferred normal vectors.
All parameters are given in \cref{table:platonic-solids:parameters}.

\begin{table}[htb]
  \centering
  \begin{tabular}{lS}
		\toprule
    assignment weight~$\assignmentWeight$ \eqref{eq:general-model-repeated}
		&
    \num{20}
		\\
    total-variation weight~$\TVweight$ \eqref{eq:general-model-repeated}
		&
    \num{0.001}
		\\
    mesh quality weight~$\meshqualityweight$ \eqref{eq:fidelity-term}
		&
    \num{e-5}
		\\
    augmentation parameter~$\augm_1$ \eqref{eq:augmented-Lagrangian:problem}
		&
    \num{1000}
		\\
    augmentation parameter~$\augm_2$ \eqref{eq:augmented-Lagrangian:problem}
		&
    \num{10}
		\\
    augmentation parameter~$\augm_3$ \eqref{eq:augmented-Lagrangian:problem}
		&
    \num{1000}
		\\
    inner product parameter~$c$ \eqref{eq:deformation-field:inner-product}
    &
    \num{0.1}
    \\
		\bottomrule
  \end{tabular}
  \caption{%
		Parameters for the platonic solids examples (\cref{figure:platonic-solids:tetrahedron} and \cref{figure:platonic-solids:dodecahedron}).
		Both examples use the same set of parameters.
	}
  \label{table:platonic-solids:parameters}
\end{table}

The results are shown in \cref{figure:platonic-solids:tetrahedron} for the tetrahedron and in \cref{figure:platonic-solids:dodecahedron} for the dodecahedron.
In order to visualize the evolution of the mesh during \cref{algorithm:ADMM}, we display~$\sequence{\mesh}{k}$ for $k = 0$ (the initial sphere mesh), $k = 50$ and $k = 500$ as intermediate results and the final mesh at convergence.

\begin{figure*}[htp]
	\centering
	\begin{subfigure}[b]{0.47\textwidth}
		\centering
		\includegraphics[width = \textwidth]{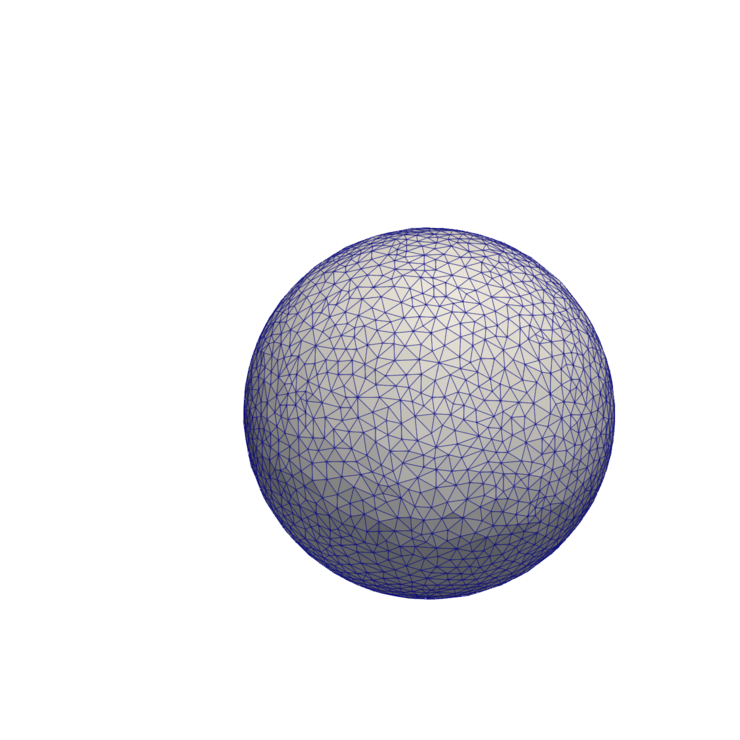}
		\caption{$k = 0$ (initial mesh)}
		\label{figure:platonic-solids:tetrahedron:input}
	\end{subfigure}
	\hfill
	\begin{subfigure}[b]{0.47\textwidth}
		\centering
		\includegraphics[width = \textwidth]{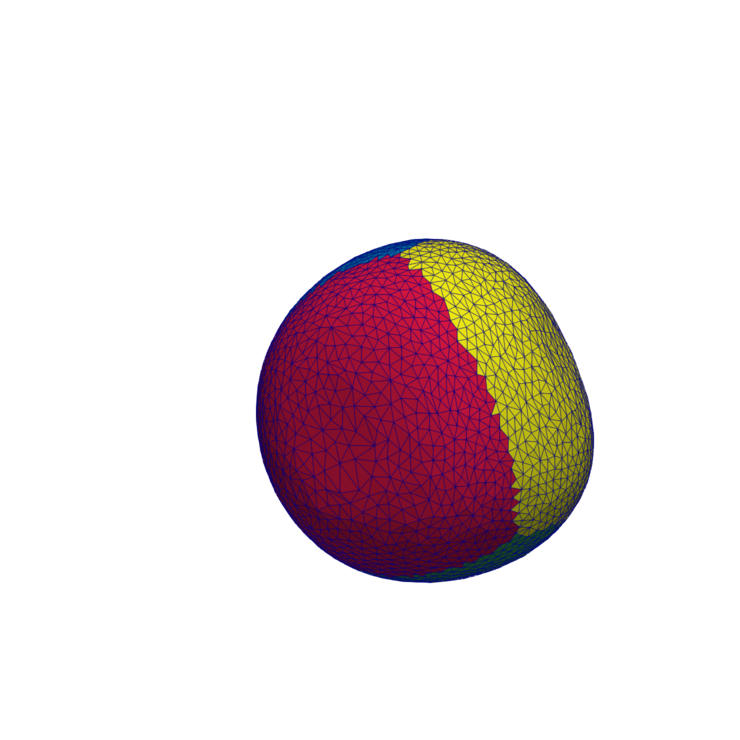}
		\caption{$k = 50$}
		\label{figure:platonic-solids:tetrahedron:early}
	\end{subfigure}
	\begin{subfigure}[b]{0.47\textwidth}
		\centering
		\includegraphics[width = \textwidth]{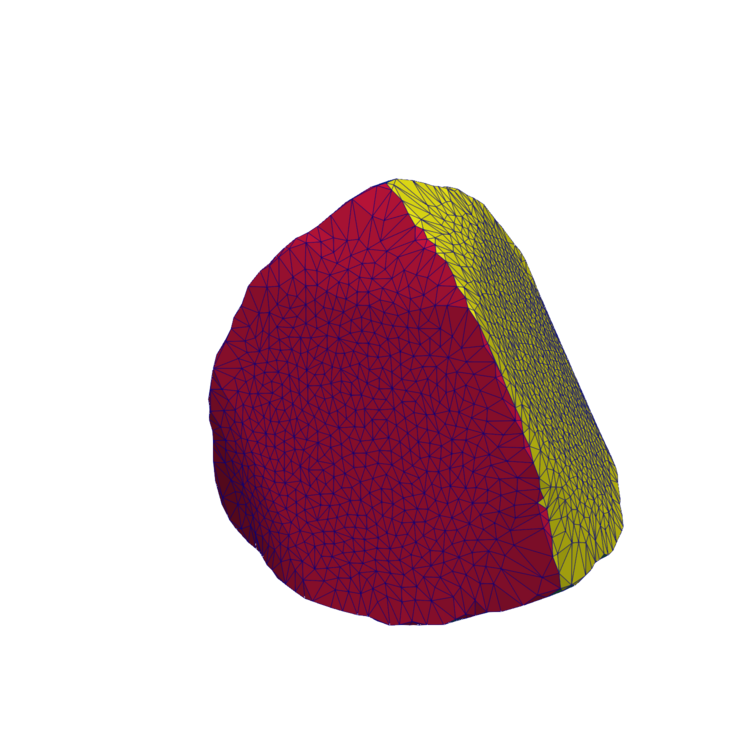}
		\caption{$k = 500$}
		\label{figure:platonic-solids:tetrahedron:late}
	\end{subfigure}
	\begin{subfigure}[b]{0.47\textwidth}
		\centering
		\includegraphics[width = \textwidth]{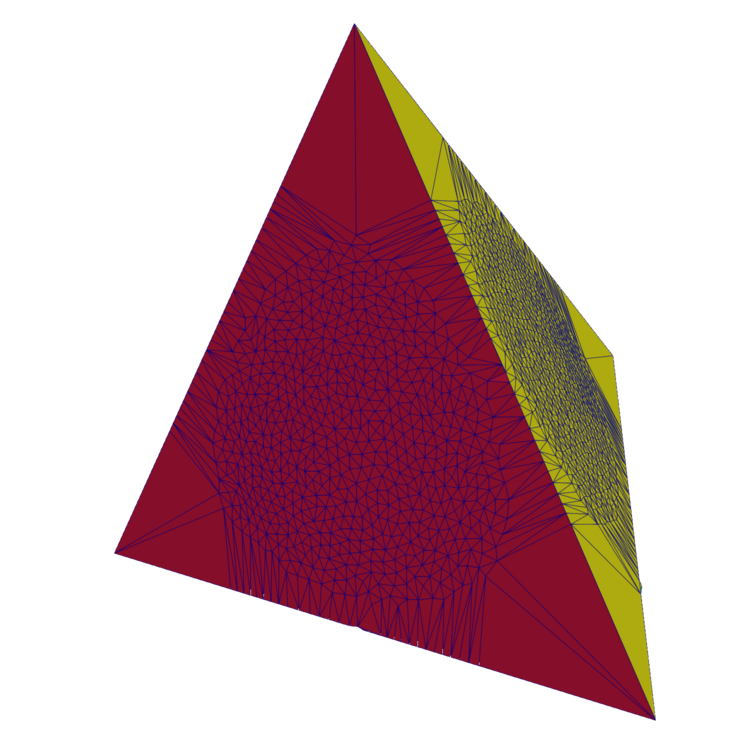}
    \caption{$k = 2117$ (final result)}
		\label{figure:platonic-solids:tetrahedron:result}
	\end{subfigure}
	\caption{%
    Iteration history of \cref{algorithm:ADMM} for the tetrahedron platonic solid problem (\cref{subsection:example:platonic-solids}, $L = 4$~labels) with a sphere as initial guess and an assignment weight~$\assignmentWeight = 20$.
		Cells are colored according to the assigned label.
	}
	\label{figure:platonic-solids:tetrahedron}
\end{figure*}

\begin{figure*}[htp]
	\centering
	\begin{subfigure}[b]{0.47\textwidth}
		\centering
		\includegraphics[width = \textwidth]{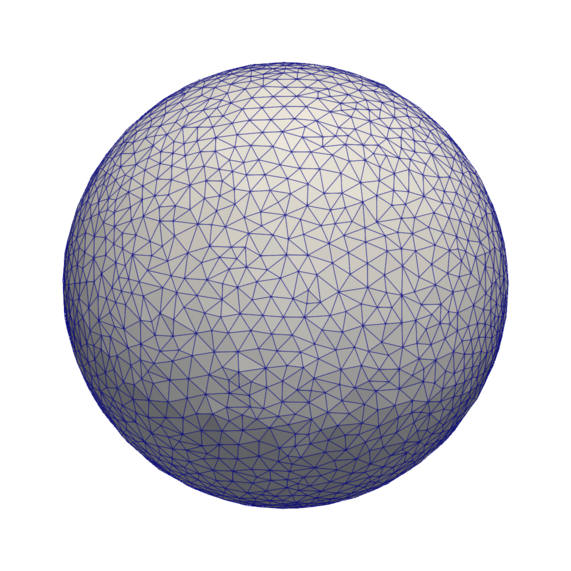}
		\caption{$k = 0$ (initial mesh)}
		\label{figure:platonic-solids:dodecahedron:input}
	\end{subfigure}
	\hfill
	\begin{subfigure}[b]{0.47\textwidth}
		\centering
		\includegraphics[width = \textwidth]{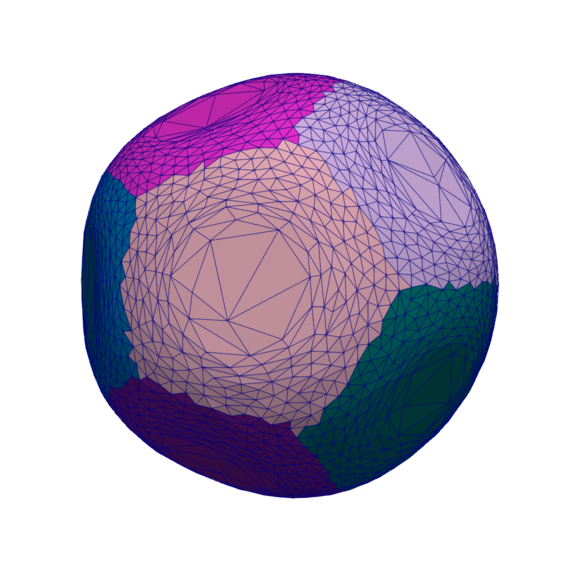}
		\caption{$k = 50$}
		\label{figure:platonic-solids:dodecahedron:early}
	\end{subfigure}
	\begin{subfigure}[b]{0.47\textwidth}
		\centering
		\includegraphics[width = \textwidth]{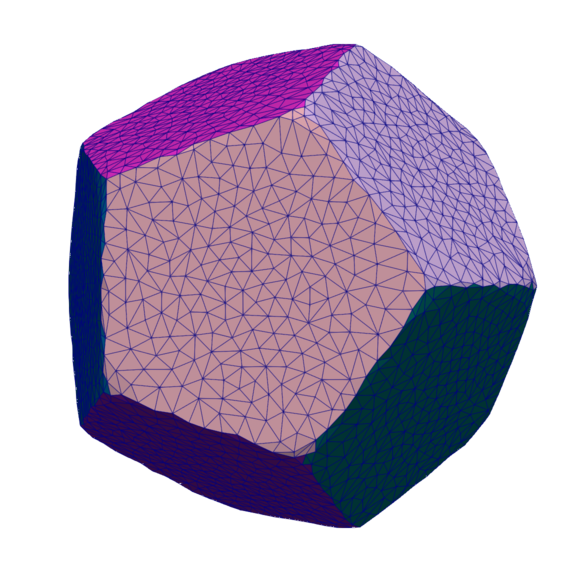}
		\caption{$k = 500$}
		\label{figure:platonic-solids:dodecahedron:late}
	\end{subfigure}
	\begin{subfigure}[b]{0.47\textwidth}
		\centering
		\includegraphics[width = \textwidth]{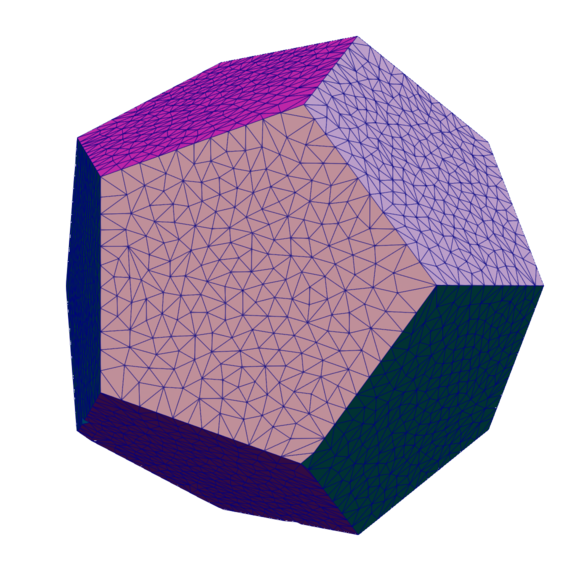}
		\caption{$k = 4202$ (final result)}
		\label{figure:platonic-solids:dodecahedron:result}
	\end{subfigure}
	\caption{%
    Iteration history of \cref{algorithm:ADMM} for the dodecahedron platonic solid problem (\cref{subsection:example:platonic-solids}, $L = 12$~labels) with a sphere as initial guess and an assignment weight~$\assignmentWeight = 20$.
		Cells are colored according to the assigned label.
	}
	\label{figure:platonic-solids:dodecahedron}
\end{figure*}

\subsection{City Skyline Denoising}
\label{subsection:example:city-skyline}

In this section, we present an example closer to a real-world application, in which preferred normal vectors can be used as a prior knowledge.
We consider a synthetically created city skyline.
On the ground surface, a $5 \times 5$ grid of buildings of variable size are placed with random heights and offsets, but aligned with the axes.
The resulting surface (\cref{figure:city-skyline:ground-truth}) is discretized into a triangle mesh with \num{6288}~triangles and \num{3146}~vertices (\cref{figure:city-skyline:ground-truth}) using \gmsh; see \cite{GeuzaineRemacle:2009:1}.
Similar to the example in \cref{subsection:example:sphere}, we add Gaussian noise with variance $\sigma^2 = 0.04 \, e^2$ to the vertex positions, yielding an input mesh with noisy vertex positions~$\vertexdata$ displayed in \cref{figure:city-skyline:noisy-data}.
In practice, such noisy data might result from a 3D scan of the city skyline, for instance using LiDAR scanning \cite{KrishnanCrosbyNandigamPhanCowartBaruArrowsmith:2011:1} or photogrammetry.

In this example, we compare two denoising methods.
We use the proposed model \eqref{eq:general-model-repeated} which we provide with the six obvious preferred normal directions, \ie, antipodal unit vectors in the directions of the axes.
We compare the results with the baseline total variation-based denoising model~\eqref{eq:problem:baseline-tv-based-denoising}; see also \cite[eq.(4.1)]{BaumgaertnerBergmannHerzogSchmidtVidalNunezWeiss:2025:1}.
Once again, we choose the objective~$\fidelity$ defined in~\eqref{eq:fidelity-term}.

The parameters $\assignmentWeight$ and $\TVweight$ for \eqref{eq:general-model-repeated} and parameter $\gamma$ for the baseline model~\eqref{eq:problem:baseline-tv-based-denoising} are chosen by a manual grid search for each model such that the resulting mesh approximately minimizes $\fidelity_1(\mesh;\groundtruth)$ as a measure of distance to the ground truth mesh.
The parameters for the optimal setup for the proposed model~\eqref{eq:general-model-repeated} are shown in \cref{table:city-skyline:parameters}.
For the baseline model~\eqref{eq:problem:baseline-tv-based-denoising}, we use $\gamma = 0.015$, $\augm = 0.1$ and $\meshqualityweight = 2 \cdot 10^{-8}$.

\begin{table}[htb]
  \centering
  \begin{tabular}{lS}
		\toprule
    assignment weight~$\assignmentWeight$ \eqref{eq:general-model-repeated}
		&
    \num{1}
		\\
    total-variation weight~$\TVweight$ \eqref{eq:general-model-repeated}
		&
    \num{e-8}
		\\
    mesh quality weight~$\meshqualityweight$ \eqref{eq:fidelity-term}
		&
    \num{e-7}
		\\
    augmentation parameter~$\augm_1$ \eqref{eq:augmented-Lagrangian:problem}
		&
    \num{12.5}
		\\
    augmentation parameter~$\augm_2$ \eqref{eq:augmented-Lagrangian:problem}
		&
    \num{1.25}
		\\
    augmentation parameter~$\augm_3$ \eqref{eq:augmented-Lagrangian:problem}
		&
    \num{12.5}
		\\
    inner product parameter~$c$ \eqref{eq:deformation-field:inner-product}
    &
    \num{0.3}
    \\
		\bottomrule
  \end{tabular}
  \caption{Parameters for the city skyline example (\cref{figure:city-skyline:preferred-normals}).}
  \label{table:city-skyline:parameters}
\end{table}

The results of this experiment are shown in \cref{figure:city-skyline}.
They clearly highlight the benefit of using the preferred normal vectors as prior knowledge on the denoising result.

\begin{figure*}[htp]
	\centering
	\begin{subfigure}[b]{0.47\textwidth}
		\centering
		\includegraphics[width = \textwidth]{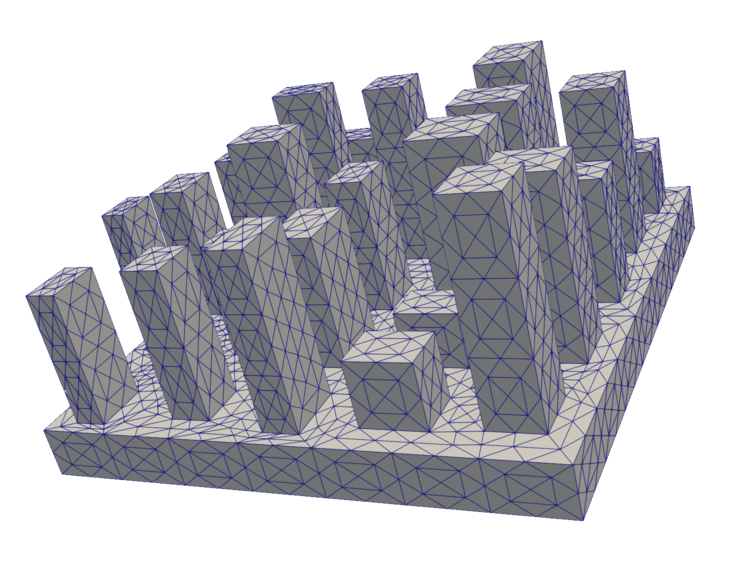}
		\caption{ground truth mesh}
		\label{figure:city-skyline:ground-truth}
	\end{subfigure}
	\hfill
	\begin{subfigure}[b]{0.47\textwidth}
		\centering
		\includegraphics[width = \textwidth]{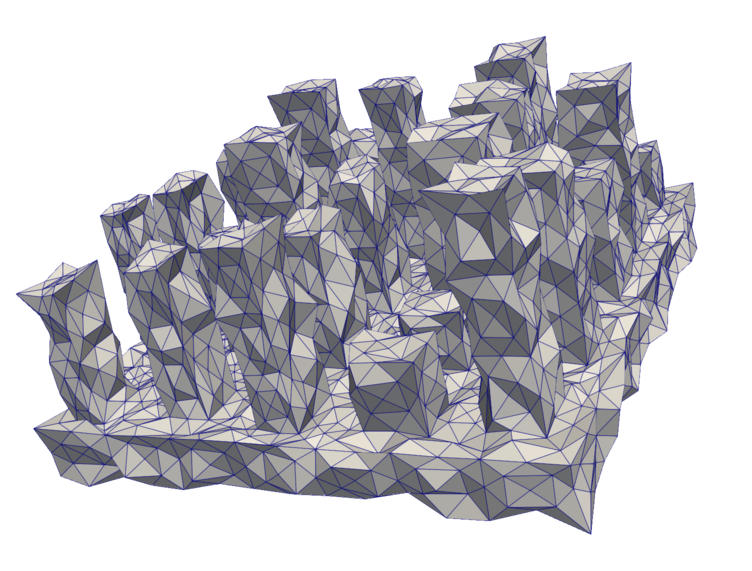}
		\caption{noisy input mesh}
		\label{figure:city-skyline:noisy-data}
	\end{subfigure}
	\begin{subfigure}[b]{0.47\textwidth}
		\centering
		\includegraphics[width = \textwidth]{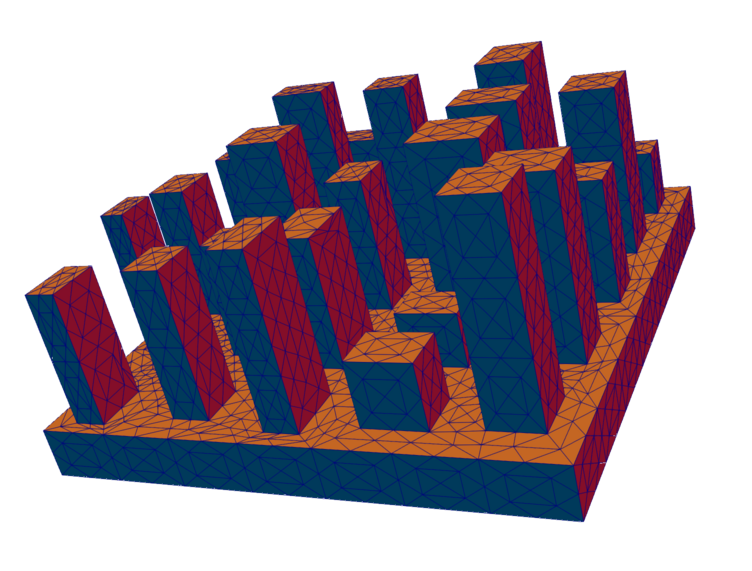}
		\caption{solution using the model \eqref{eq:general-model-repeated} with $\assignmentWeight = 1$ and $\TVweight = 10^{-8}$}
		\label{figure:city-skyline:preferred-normals}
	\end{subfigure}
	\hfill
	\begin{subfigure}[b]{0.47\textwidth}
		\centering
		\includegraphics[width = \textwidth]{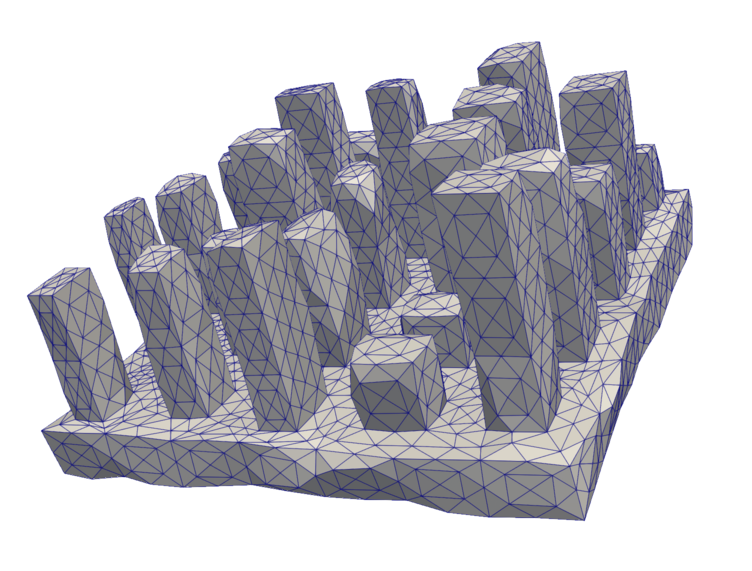}
    \caption{solution using the baseline TV model~\eqref{eq:problem:baseline-tv-based-denoising} with $\TVweight = 0.015$}
		\label{figure:city-skyline:denoising}
	\end{subfigure}
	\caption{Denoising of the city skyline geometry (\cref{subsection:example:city-skyline}) with the proposed model \eqref{eq:general-model-repeated} (bottom left) and the baseline total variation-based model \eqref{eq:problem:baseline-tv-based-denoising} (bottom right), which is not informed by preferred normal directions.}
	\label{figure:city-skyline}
\end{figure*}

The proposed model \eqref{eq:general-model-repeated} can recover the ground truth mesh with $\fidelity_1(\mesh;\groundtruth) = 0.138$, whereas the baseline TV denoising model without any additional normal direction information~\eqref{eq:problem:baseline-tv-based-denoising} results in a distance of $\fidelity_1(\mesh;\groundtruth) = 0.755$ and a visibly inferior outcome (\cref{figure:city-skyline:denoising}).
While the latter model also identifies regions of constant normal directions, those are not necessarily aligned with the coordinate axes.

\subsection{Stanford Bunny}
\label{subsection:example:stanford-bunny}

In this example, we apply problem \eqref{eq:general-model-repeated} for an artistic purpose, namely to modify a geometry to create a rough wood carving effect.
This is achieved by specifying a relatively small set of preferred normal directions distributed evenly around a sphere.
For this experiment, we consider the Stanford bunny mesh \cite{stanford-bunny-source} with \num{69427}~triangles and \num{34820}~vertices, preprocessed using \meshlab version~2023.12, \cite{CignoniCallieriCorsiniDellepianeGanovelliRanzuglia:2008:1} and \meshio \cite{Schloemer:2024:1} to remove all non-conforming triangles and improve the overall quality of the initial mesh.
No noise is added to the vertex positions~$\vertexdata$.
We consider two different label sets featuring $L = 20$ and $L = 30$ labels, respectively, distributed uniformly around the sphere by means of Fibonacci lattices.
Moreover, we set $\assignmentWeight = 100$ so that the triangles are strongly incentivized to follow the prescribed normal directions.
The remaining parameters are listed in \cref{table:bunny:parameters}.

\begin{table}[htb]
  \centering
  \begin{tabular}{lS}
		\toprule
    assignment weight~$\assignmentWeight$ \eqref{eq:general-model-repeated}
		&
    \num{100}
		\\
    total-variation weight~$\TVweight$ \eqref{eq:general-model-repeated}
		&
    \num{0.05}
		\\
    mesh quality weight~$\meshqualityweight$ \eqref{eq:fidelity-term}
		&
    \num{e-11}
		\\
    augmentation parameter~$\augm_1$ \eqref{eq:augmented-Lagrangian:problem}
		&
    \num{10000}
		\\
    augmentation parameter~$\augm_2$ \eqref{eq:augmented-Lagrangian:problem}
		&
    \num{100}
		\\
    augmentation parameter~$\augm_3$ \eqref{eq:augmented-Lagrangian:problem}
		&
    \num{10000}
		\\
    inner product parameter~$c$ \eqref{eq:deformation-field:inner-product}
    &
    \num{0.1}
    \\
		\bottomrule
  \end{tabular}
	\caption{Parameters for the Stanford bunny example (\cref{figure:bunny}).}
  \label{table:bunny:parameters}
\end{table}

The results are shown in \cref{figure:bunny}.

\begin{figure}[htb]
	\centering
	\begin{subfigure}[b]{0.48\textwidth}
		\centering
		\includegraphics[width = \textwidth]{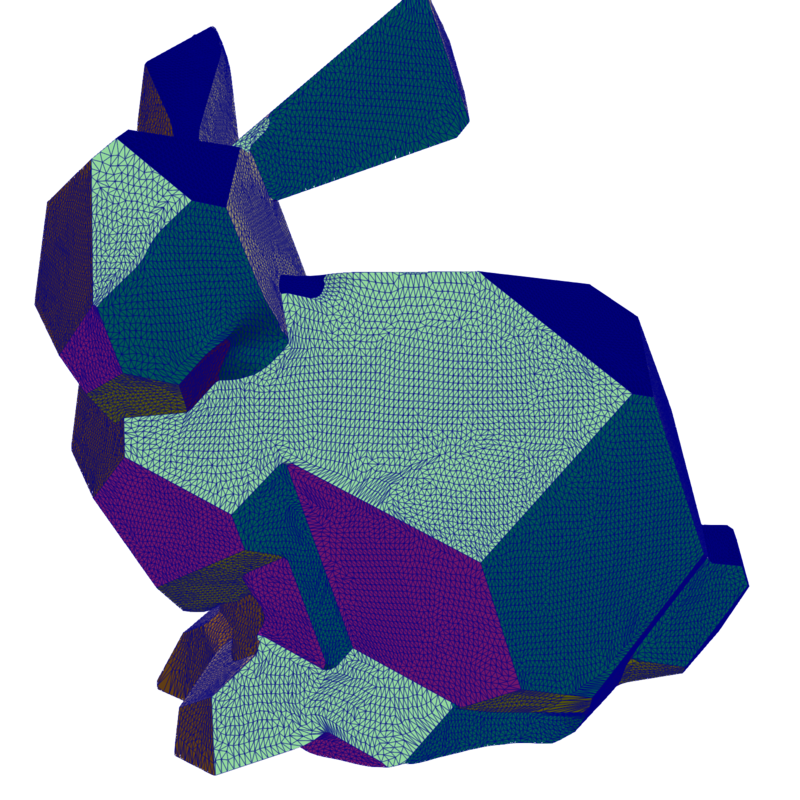}
		\caption{$L = 20$}
		\label{figure:bunny:L=20}
	\end{subfigure}
	\hfill
	\begin{subfigure}[b]{0.48\textwidth}
		\centering
		\includegraphics[width = \textwidth]{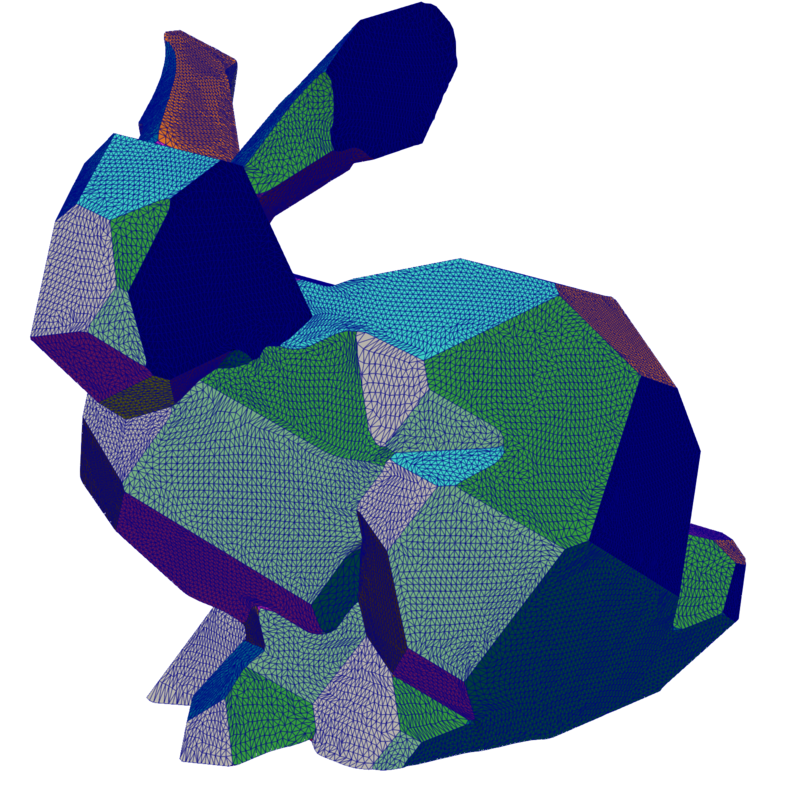}
		\caption{$L = 30$}
		\label{figure:bunny:L=30}
	\end{subfigure}
	\caption{%
		Solutions for problem \eqref{eq:general-model-repeated} for the Stanford bunny and different numbers~$L$ of preferred normal vectors.
		Cells are colored according to the assigned label.
	}
	\label{figure:bunny}
\end{figure}

\subsection{Inscription from Monastery Maulbronn}
\label{subsection:example:inscription}

This example aims to show the applicability of the proposed model \eqref{eq:general-model-repeated} to real-world data.
We consider a 3D scan of an inscription on a gravestone dated 1377, located at the monastery in Maulbronn, a UNESCO world heritage site in Germany; see \cite[p.~790]{Untermann:2024:1}.
The text is written in Gothic minuscules, which means that the carved letters exhibit a limited number of normal directions:
the letter outlines (as seen in 2D top-view projection) are either aligned with the coordinate axis, or they are diagonal at \SI{45}{\degree} angles.
An analysis shows that the letters' facets are also inclined at an approximate angle of \SI{22.5}{\degree} from the vertical axis.
Hence, we use the following set of $L = 9$ preferred normal vectors (see \cref{figure:inscription:labels})
\begin{equation}
  \label{equation:inscription:preferred-normals}
	\labels_\ell
  =
	\begin{pmatrix}
		\sin \paren[auto](){\ell \cdot \frac{2 \pi}{8}}
    \cdot
    \sin \paren[auto](){\theta}
		\\
		\cos \paren[auto](){\ell \cdot \frac{2 \pi}{8}}
    \cdot
    \sin \paren[auto](){\theta}
		\\
    \cos \paren[auto](){\theta}
	\end{pmatrix}
	\text{ for }
	\ell = 1, \ldots, 8
  ,
	\quad
	\labels_{9}
	=
	\begin{pmatrix}
		0
		\\
		0
		\\
		1
	\end{pmatrix}
	,
\end{equation}
where $\theta = \frac{\pi}{8}$.
The normal vector $\labels_{9}$ is aligned with the $z$-axis and represents the flat area with no inscription.

\begin{figure}[htb]
	\centering
  \includegraphics[width = 0.3\textwidth]{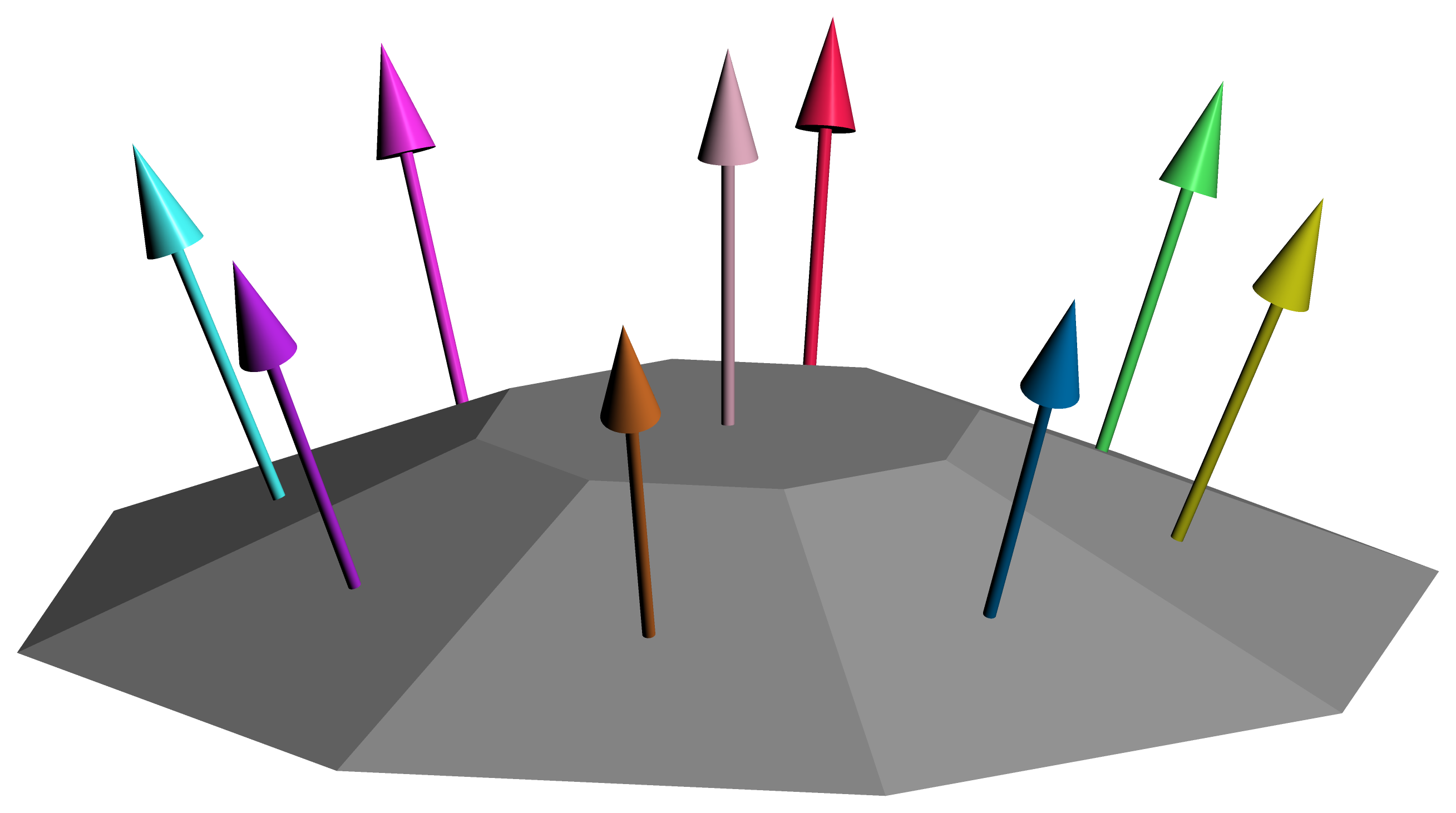}
  \caption{%
		Visualization of the label set \eqref{equation:inscription:preferred-normals} for the inscription example (\cref{subsection:example:inscription}).
	}
	\label{figure:inscription:labels}
\end{figure}

\begin{figure}[htb]
	\centering
	\begin{tikzpicture}
		\node[anchor=south west, inner sep=0] (image) at (0,0)
			{\includegraphics[width=0.8\textwidth]{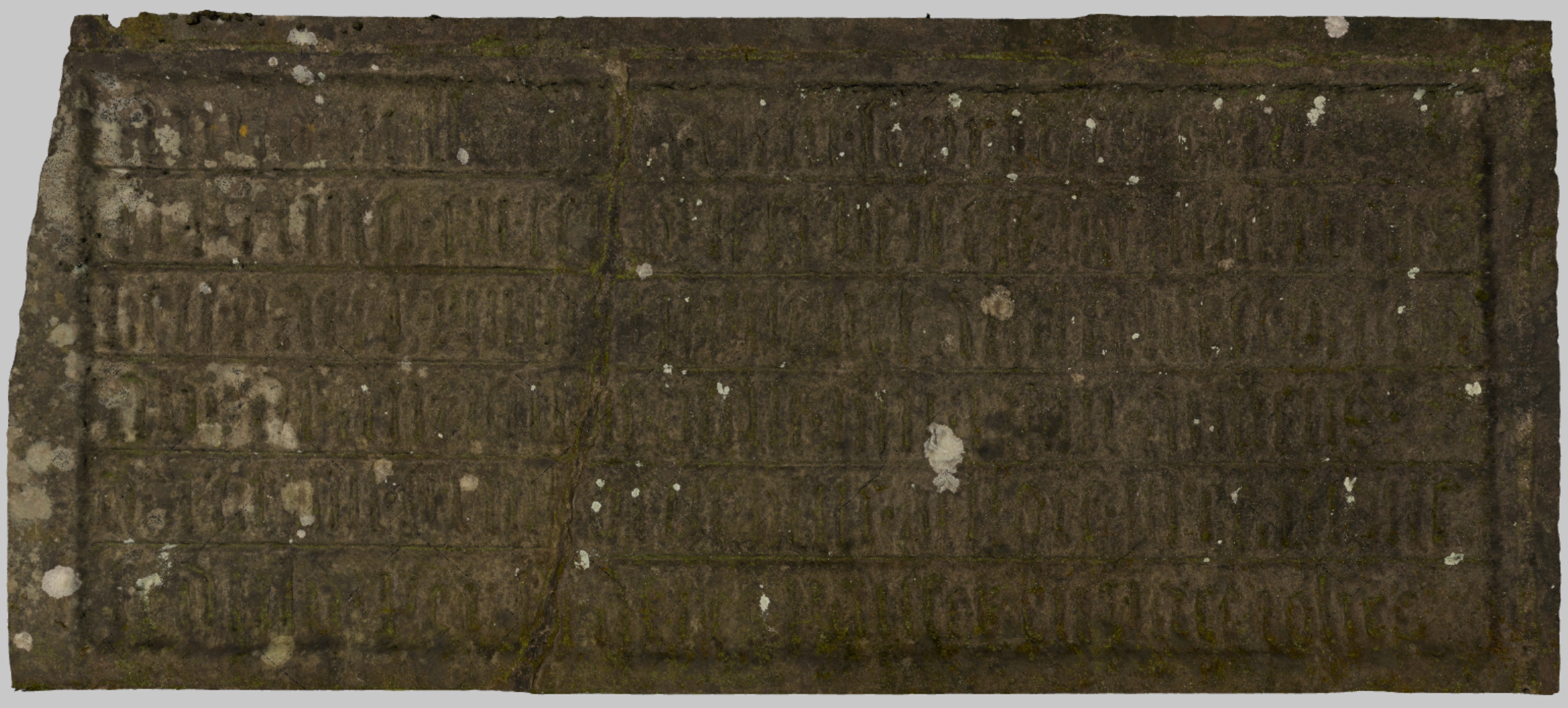}};
		\begin{scope}[x={(image.south east)}, y={(image.north west)}]
			\draw[red, very thick] (0.85,0.45) rectangle (0.97,0.77);
		\end{scope}
	\end{tikzpicture}
	\caption{Photo of the inscription at the monastery in Maulbronn, Germany.
	The red rectangle indicates the part of the inscription that is used for the denoising experiment in \cref{subsection:example:inscription}.}
	\label{figure:inscription:photo}
\end{figure}

We limit this numerical experiment to part of the full inscription.
The relevant part is shown in \cref{figure:inscription:photo} and it is represented by a mesh containing \num{321865}~vertices.
All parameters can be found in~\cref{table:inscription:parameters}.

\begin{table}[htb]
  \centering
  \begin{tabular}{lS}
		\toprule
    assignment weight~$\assignmentWeight$ \eqref{eq:general-model-repeated}
		&
    \num{5}
		\\
    total-variation weight~$\TVweight$ \eqref{eq:general-model-repeated}
		&
    \num{0.9}
		\\
    mesh quality weight~$\meshqualityweight$ \eqref{eq:fidelity-term}
		&
    \num{e-4}
		\\
    augmentation parameter~$\augm_1$ \eqref{eq:augmented-Lagrangian:problem}
		&
    \num{20}
		\\
    augmentation parameter~$\augm_2$ \eqref{eq:augmented-Lagrangian:problem}
		&
    \num{20}
		\\
    augmentation parameter~$\augm_3$ \eqref{eq:augmented-Lagrangian:problem}
		&
    \num{20}
		\\
    inner product parameter~$c$ \eqref{eq:deformation-field:inner-product}
    &
    \num{0.1}
    \\
		\bottomrule
  \end{tabular}
	\caption{Parameters for the inscription example (\cref{figure:inscription}).}
  \label{table:inscription:parameters}
\end{table}

\begin{figure}[htp]
	\centering
	\begin{subfigure}[b]{0.32\textwidth}
		\centering
		\includegraphics[width = \textwidth]{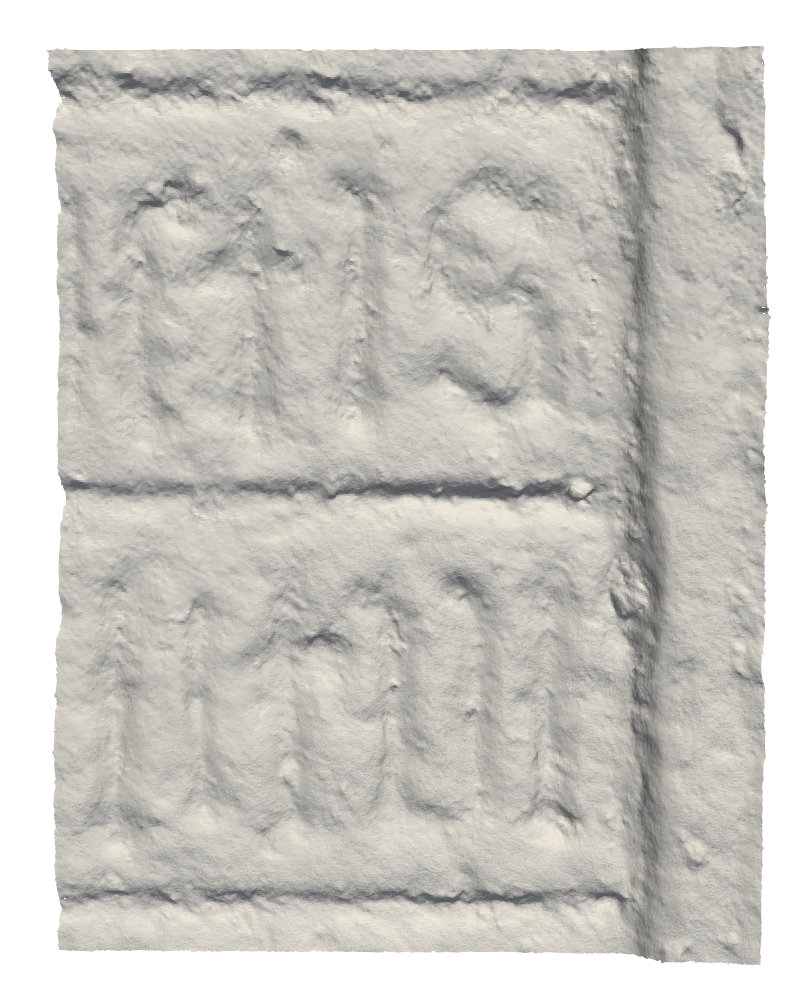}
	\end{subfigure}
	\begin{subfigure}[b]{0.32\textwidth}
		\centering
		\includegraphics[width = \textwidth]{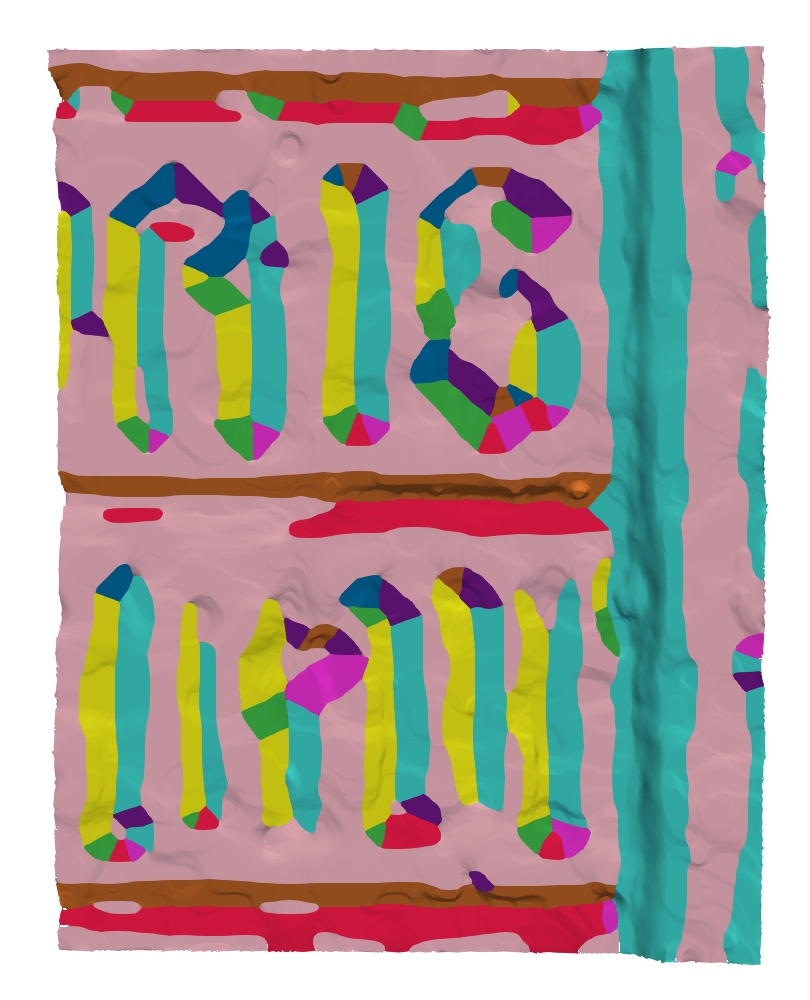}
	\end{subfigure}
	\begin{subfigure}[b]{0.32\textwidth}
		\centering
		\includegraphics[width = \textwidth]{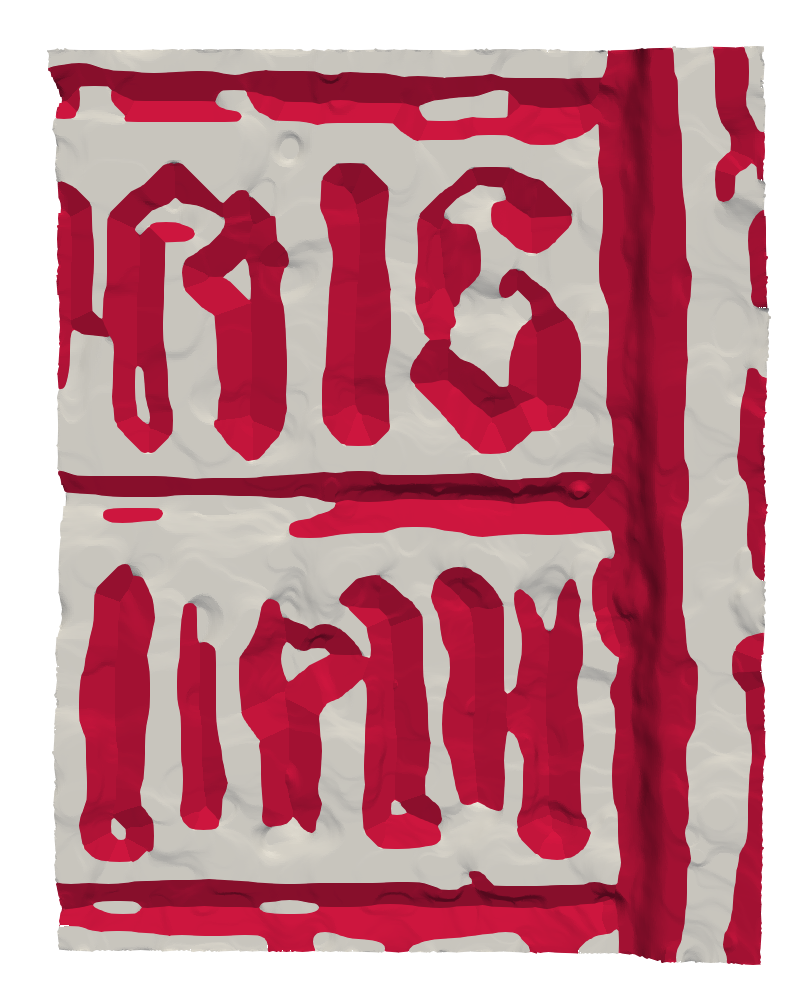}
	\end{subfigure}
  \caption{%
		Denoising of the inscription (\cref{subsection:example:inscription}) with the proposed model \eqref{eq:general-model-repeated} and model parameters defined in \cref{table:inscription:parameters}.
		Left: noisy input mesh.
		Center: denoised solution using model \eqref{eq:general-model-repeated} showing all labels in with distinct colors; see \eqref{equation:inscription:preferred-normals} and \cref{figure:inscription:labels}.
		Right: the same solution showing labels~$\labels_{1, \ldots, 8}$ in red and $\labels_{9}$ in gray.
		Triangles are colored according to their assigned label directions.
		The letters read \enquote{ctis} (top line), and \enquote{litur} (bottom line).
	}
	\label{figure:inscription}
\end{figure}

We observe that the model \eqref{eq:general-model-repeated} is able to denoise the inscription and help recover the carved letters while preserving the flat areas of the stone surface.
Some difficulties can be observed with the letter~\enquote{s} in the top right.

\section{Conclusion}
\label{section:conclusion}

In this paper, we presented a novel model \eqref{eq:general-model-repeated} for denoising triangular surface meshes based on a set of preferred normal directions.
The idea is to couple the denoising problem with an assignment problem, where the assignment function $\assign$ assigns to each triangle one of the preferred normal directions.
We presented an ADMM scheme that solves the model problem \eqref{eq:general-model-repeated} by splitting the problem into easy-to-solve subproblems.
The shape update step (\cref{subsection:admm:shape-optimization-problem}), which is computationally the most expensive subproblem, is solved efficiently using Newton's method with a truncated conjugate gradient method for the arising linear systems.
We demonstrated that the total-variation penalty on the assignment function can be used to control the size of the flat regions of constant assignment.
It was also demonstrated that the additional information provided by preferred normal directions can significantly improve the denoising result.
The final example (\cref{subsection:example:inscription}) shows that the proposed model can be applied in the context of archeological analysis of 3D scans, where further information about the scanned geometry is known from context a priori.

\section*{Acknowledgements}

We gratefully thank Steffen Bauer (Heidelberg University, \url{https://vngg.iwr.uni-heidelberg.de/People/bauer/}) for providing the 3D scan of the inscription from Monastery Maulbronn used in \cref{subsection:example:inscription} and helpful discussions.

\appendix

\printbibliography

\end{document}